\documentclass[11pt]{article}
\usepackage{dilab_arxiv}
\usepackage{amsmath,amsfonts,bm}

\def\eqref#1{equation~\ref{#1}}
\def\1{\bm{1}}

\DeclareMathAlphabet{\mathsfit}{\encodingdefault}{\sfdefault}{m}{sl}
\SetMathAlphabet{\mathsfit}{bold}{\encodingdefault}{\sfdefault}{bx}{n}

\usepackage{enumitem}
\usepackage{flafter}
\usepackage{placeins}
\usepackage{bookmark}
\apptocmd{\thebibliography}{\raggedright}{}{}

\newtheorem{theorem}{Theorem}[section]
\newtheorem{assumption}[theorem]{Assumption}
\newtheorem{proposition}[theorem]{Proposition}
\newtheorem{lemma}[theorem]{Lemma}
\newtheorem{corollary}[theorem]{Corollary}
\newtheorem{definition}[theorem]{Definition}
    
\setlist[itemize]{leftmargin=*,itemsep=3pt,topsep=5pt}

\hypersetup{
  pdftitle={The Router Within: Eliciting Native Skill Routing from a Frozen LLM},
  pdfauthor={Ruishuo Chen, Xun Wang, Yu Chen, Zhuoran Li, and Longbo Huang},
  pdfsubject={Gavel: skill routing from a frozen language model},
  pdfkeywords={large language models, agents, skills, routing, Gavel},
  bookmarksnumbered=true
}

\usepackage{wrapfig}
\usepackage[ruled,vlined,linesnumbered]{algorithm2e}
\usepackage{subcaption}
\usepackage{array}
\usepackage{graphicx}
\usepackage{makecell}

\title{Towards Understanding the Training of GFlowNets: A Theoretical Perspective}
\runningtitle{Towards Understanding the Training of GFlowNets: A Theoretical Perspective}
\date{arXiv preprint, October 2026}
\paperlogo{\includegraphics[height=1.5cm]{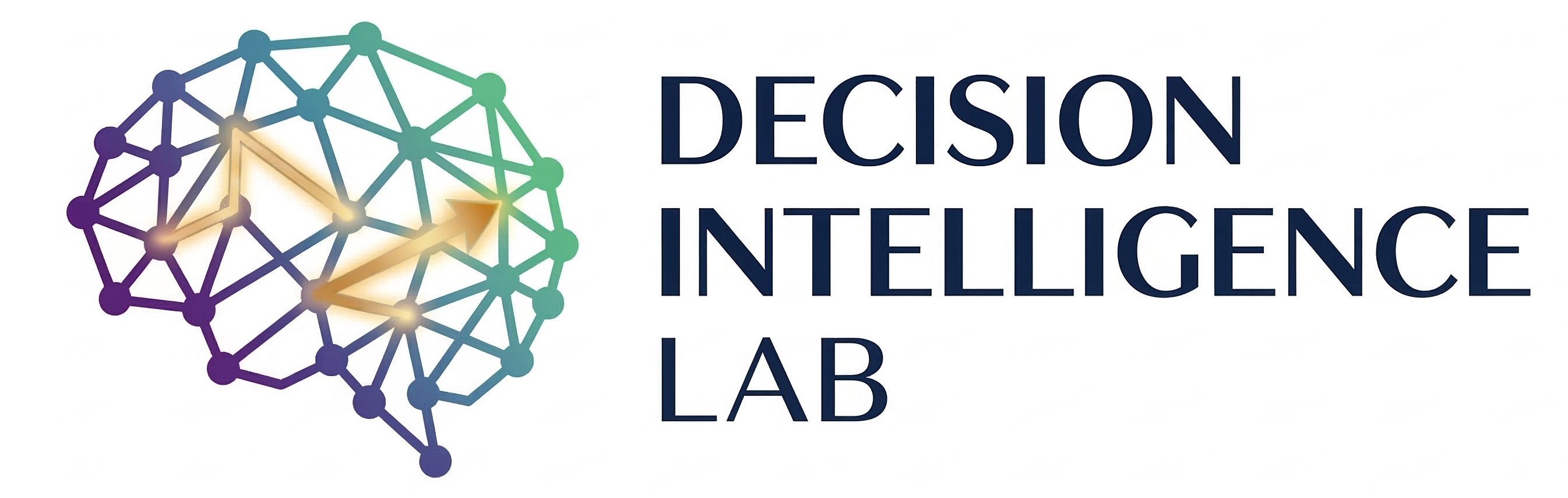}}
\author{
  Rui Hu, Yu Chen, Zhuoran Li, and
  Longbo Huang$^{\text{\faEnvelope}}$
  \\[0.3em]\normalfont
  Institute for Interdisciplinary Information Sciences, Tsinghua University
  \\
  \text{\faEnvelope}\ Correspondence: \href{mailto:longbohuang@tsinghua.edu.cn}{\texttt{longbohuang@tsinghua.edu.cn}}
}

\begin{document}
\maketitle
\thispagestyle{fancy}

\begin{abstract}
Generative Flow Networks (GFlowNets) learn policies that sample discrete objects in proportion to a positive reward. They are trained by minimizing balance objectives, whose exact satisfaction guarantees correct sampling. However, a comprehensive theory linking nonzero balance loss and approximate stationarity to sampling accuracy, especially under evolving sampling distributions and backward policies, is still lacking. To bridge this gap, we develop a unified theory for Flow Matching, Detailed Balance, and Trajectory Balance in this regime. We first bound terminal symmetric KL by the square root of the expected balance loss. Then, we show that approximate stationarity controls this distribution error up to a tangent-space residual measuring parameterization incompleteness. Finally, we establish finite-time convergence under fixed or controlled sampling and fixed, balance-trained, or separately updated backward policies. Under standard smoothness assumptions, with fixed or controlled sampling and a fixed or balance-trained backward policy, the expected loss is $O(\bar\epsilon^2+T^{-1})$ and expected terminal symmetric KL is $O(\bar\epsilon+T^{-1/2})$ under uniform coverage. For a separately updated backward policy, we characterize the sufficient update decay to approach the residual neighborhood and preserve these rates. These results connect balance objectives to sampling accuracy and clarify how parameterization, coverage, and moving components govern GFlowNet training.
\end{abstract}

\section{Introduction}

%Many important scientific and engineering problems require more than finding a single high-scoring object: one often needs a diverse collection of candidates that captures several distinct modes of a complex design space.  Generative Flow Networks (GFlowNets; \citealt{bengio2021flow,bengio2023gflownet}) address this need by learning to sample an object with probability proportional to a positive reward.  High-reward objects are therefore sampled frequently, while different high-reward regions retain probability mass. This combination of quality and diversity has made GFlowNets useful in important applications such as molecule generation \citep{zhu2023sample, koziarski2024rgfn}, combinatorial optimization \citep{zhang2023let}, probabilistic inference \citep{manta2023gflownets, deleu2024discrete}, text generation \citep{hu2024amortizing}, and image generation \citep{liu2025efficient, yun2025learning}.

Generative Flow Networks (GFlowNets, \citet{bengio2021flow, bengio2023gflownet}) have emerged as a promising framework for training generative models that sample objects with probabilities proportional to a positive reward. High-reward objects are therefore sampled frequently, while different high-reward regions retain probability mass. This combination of quality and diversity has made GFlowNets useful in various applications such as molecule generation \citep{zhu2023sample, koziarski2024rgfn}, combinatorial optimization \citep{zhang2023let}, probabilistic inference \citep{manta2023gflownets, deleu2024discrete}, text generation \citep{hu2024amortizing}, and image generation \citep{liu2025efficient, yun2025learning}.

GFlowNets transform the sampling problem into a sequential generation process on a directed acyclic graph (DAG). Starting from an initial state, a forward policy constructs an object through a series of local decisions until it reaches a terminal state. The central idea is to view probability mass as flow through the DAG. If the flow is conserved at intermediate states and the terminal flow matches the reward, then the forward policy samples terminal states from the desired reward-proportional distribution. This global property can be enforced through local or trajectory-level balance conditions, leading to the standard Flow Matching (FM, \citealt{bengio2021flow}), Detailed Balance (DB, \citealt{bengio2023gflownet}), and Trajectory Balance (TB, \citealt{malkin2022trajectory}) objectives.

This elegant construction explains what happens when every balance condition is satisfied exactly, but it does not explain what happens during practical training. A parameterized model may attain only a small nonzero loss or approach a stationary point rather than a global minimizer. Moreover, the states, edges, or trajectories used for training may be drawn from an off-policy and time-varying distribution, while the backward policy may be fixed, trained jointly with the balance loss, or updated using a separate criterion. These choices make the objective seen by stochastic gradient descent (SGD) evolve over time. Consequently, to understand GFlowNet training, a systematic study is required to connect loss and stationarity to terminal accuracy while accounting for these moving components. This motivates three essential questions:
\begin{enumerate}
    \item \textit{Does a small balance loss guarantee that the learned terminal distribution is close to the reward-proportional target?}
    \item \textit{When optimization reaches an approximate stationary point, how good is the resulting sampler, and what limits its accuracy?}
    \item \textit{As sampling strategies and backward policies are adapted during training, when can learning still approach the target distribution, and at what rate?}
\end{enumerate}

% Existing studies provide complementary pieces of this picture.  For example, one line of work studies generalization for TB-GFlowNets and DB-GFlowNets \citep{krichel2024on, silva2025generalization}, while another examines the sensitivity of the induced distribution to flow imbalances \citep{silva2025when}.  Other analyses connect GFlowNets to entropy-regularized reinforcement learning \citep{tiapkin2024generative} or hierarchical variational inference \citep{malkin2023gflownets}. The most relevant work is a recent study by \citet{yu2025secrets}, which establishes local convergence rates for the gradients of modified FM and DB objectives, as well as a sample complexity bound for the modified TB objective. Although these perspectives provide useful information, they do not create a coherent path linking the commonly optimized FM, DB, and TB losses to terminal distribution accuracy, stationary point quality, and finite-time training guarantees under changing sampling and backward policies. 

Existing studies provide complementary pieces of this picture. Flow-imbalance analyses \citep{silva2025when} examine how balance violations affect the generated distribution, while generalization analyses \citep{krichel2024on,silva2025generalization} study learning from limited observations. Connections to variational inference and entropy-regularized reinforcement learning \citep{malkin2023gflownets,tiapkin2024generative}, together with advances in loss design \citep{hu2025beyond,fawkes2026trajectory}, relate training objectives to distribution discrepancies. However, these relationships do not by themselves connect balance objectives to sampling accuracy. \citet{yu2025secrets} establishes gradient-convergence results for FM and DB, as well as sample-complexity guarantees for TB. However, the objectives in their analysis differ from the commonly used squared log-ratio loss and do not clarify the quality of parameterized stationary points or the behavior of stochastic training under varying sampling and backward policies. Very recently, \citet{lei2026stable} derived total variation guarantees from a uniform residual bound, providing a quantitative measure of sampling accuracy based on balance loss. A comprehensive literature review is available in Appendix~\ref{app:related-work}.

To answer the above questions, we develop a unified theory for FM, DB, and TB that connects balance loss, stationary-point quality, and finite-time sampling accuracy. Our theoretical results are summarized as follows:
\begin{enumerate}
    \item \textbf{Balance loss controls sampling accuracy.}
    For FM, DB, and TB, we bound the terminal symmetric KL by the square root of the expected balance loss, with explicit occupancy-mass and coverage factors. Thus a small loss is informative when the training samples cover the states, edges, or trajectories on which the forward and backward flows disagree.

    \item \textbf{Stationarity implies approximate optimality.}
    We decompose the log-flow imbalance into a gradient-visible tangent component and an orthogonal parameterization residual. This yields an approximate gradient-domination property in which the terminal distribution error is controlled by the gradient norm and a precisely defined parameterization-incompleteness term. Thus, approximate stationarity leads to a nearly accurate sampler.

    \item \textbf{Finite-time convergence with moving training components.}
    We analyze stochastic GFlowNet training when the sampling distribution is fixed or changes in a controlled manner and when the backward policy is fixed, jointly optimized by the balance loss, or updated by a general rule. With fixed or controlled sampling and a fixed or balance-trained backward policy, the expected balance loss is bounded by $O(\bar\epsilon^2+T^{-1})$, which gives expected terminal symmetric KL of $O(\bar\epsilon+T^{-1/2})$ under uniform coverage. For a separately updated backward policy, we characterize how fast its update must decay to ensure convergence and to preserve these rates.
\end{enumerate}

The remainder of the paper is organized as follows. Section~\ref{sec:preliminaries} introduces FM, DB, and TB in a unified framework. Section~\ref{sec:objectives} relates their expected balance losses to terminal-distribution mismatch. Section~\ref{sec:stationary} establishes the stationarity-to-optimality guarantee, and Section~\ref{sec:convergence-rates} gives the finite-time convergence analysis. Additionally, we extend the analysis to Sub-Trajectory Balance in Appendix~\ref{app:subtb}.

% Together, these results explain when optimization progress certifies sampling progress and quantify how coverage, parameterization, and update stability govern that connection. Sections~\ref{sec:preliminaries}--\ref{sec:convergence-rates} present the unified formulation and the three main guarantees in order. Appendix~\ref{app:subtb} extends the analysis to Sub-Trajectory Balance and discusses its relation to forward-looking parameterizations.

\section{Preliminaries}
\label{sec:preliminaries}

Let $\mathcal X$ be a discrete space with positive reward $R:\mathcal X\to\mathbb R_{>0}$. GFlowNet aims to sample $x\in\mathcal X$ with probability $R(x)/Z$, where $Z:=\sum_{x\in\mathcal X}R(x)$ is assumed to be finite. The generation takes place on a DAG $\mathcal G=(\mathcal S,\mathcal E)$ with a unique initial state $s_0$. Every state is reachable from $s_0$, and the set of terminal states coincides with $\mathcal X$.  We write $\mathcal T$ for the set of complete trajectories $\tau=(s_0,s_1,\ldots,s_{n(\tau)})$, where $n(\tau)$ is the number of edges and $s_{n(\tau)}\in\mathcal X$. 

A \textbf{flow network} is specified by a positive \textbf{edge flow} function $F:\mathcal E\to\mathbb R_{>0}$ satisfying flow conservation at every nonterminal state other than $s_0$:
\[
    \sum_{s':(s'\to s)\in\mathcal{E}}F(s'\to s)=\sum_{s'':(s\to s'')\in\mathcal{E}}F(s\to s'').
\]
Its \textbf{state flow} $F(s)$ is defined as the total flow going through state $s$. The network is \textbf{balanced} when $F(x)=R(x)$ for every terminal state, in which case $F(s_0)=Z$ and the induced terminal distribution is $F(x)/F(s_0)=R(x)/Z$.

A flow network induces a \textbf{forward policy} and a \textbf{backward policy} through
\[
    P_F(s'\mid s):=\frac{F(s\to s')}{F(s)}, \qquad
    P_B(s\mid s'):=\frac{F(s\to s')}{F(s')}.
\]
% During training, parameterized flows need not satisfy conservation. The following objectives penalize these imbalances while still defining normalized policies. All parameterized flows and policy probabilities on allowed edges are strictly positive and differentiable in their parameters. We consider Flow Matching (FM), Detailed Balance (DB), and Trajectory Balance (TB).

GFlowNets have different variants, defined by how the flows are parameterized and the corresponding balance losses resulting from different balance conditions. In this paper, we focus on three typical types: Flow Matching (FM), Detailed Balance (DB), and Trajectory Balance (TB). %Sub-Trajectory Balance (SubTB; \citealt{madan2023learning}) interpolates between single-edge and complete-trajectory credit assignment. Appendix~\ref{app:subtb} extends all three parts of our analysis---distribution matching, stationarity, and stochastic convergence---to this objective.

Throughout the discussion, $\bm\theta\in\Theta\subseteq\mathbb R^d$ collects the parameters optimized by the balance loss, while $\bm\varphi$ denotes the parameters of a backward policy when it is specified separately. For FM, both policies are induced by the edge flows and there is no separate backward block. For TB and DB, when the backward policy is jointly trained by the balance loss, its parameters are included in $\bm\theta$; every occurrence of $P_B(\cdot;\bm\varphi)$ below is then read as $P_B(\cdot;\bm\theta)$. 

\paragraph{Flow Matching GFlowNets\citep{bengio2021flow}.}
An FM-GFN uses a positive edge flow function $F(s\to s';\bm\theta)$. For $s\in\mathcal S\setminus \{s_0\}$, define the forward flow and the backward flow as
\[
    F^{\rightarrow}(s):=\sum_{s':(s'\to s)\in\mathcal{E}}F(s'\to s;\bm{\theta}),\quad
    F^{\leftarrow}(s):=\begin{cases}
        R(s) & , s\in\mathcal{X}\\
        \sum_{s'':(s\to s'')\in\mathcal{E}}F(s\to s'';\bm{\theta}) & , \text{otherwise}
    \end{cases},
\]
By extending the definition of $F^{\leftarrow}(s)$ to the initial state $s_0$, these edge flows induce normalized policies even before the balance conditions hold:
\[
    P_F(s'\mid s)=\frac{F(s\to s';\bm\theta)}{F^{\leftarrow}(s)}, \qquad
    P_B(s\mid s')=\frac{F(s\to s';\bm\theta)}{F^{\rightarrow}(s')}.
\]
FM enforces the balance condition $F^{\rightarrow}(s)=F^{\leftarrow}(s)$ on every non-initial state $s\neq s_0$ using the pointwise loss
\[
    L_{\mathrm{FM}}(s):=\frac{1}{2}\left(\log\frac{F^{\rightarrow}(s)}{F^{\leftarrow}(s)}\right)^2.
\]

\paragraph{Detailed Balance GFlowNets\citep{bengio2023gflownet}.}
A DB-GFN uses a positive state flow function $F(s;\bm\theta)$ on nonterminal states, a forward policy $P_F(s'\mid s;\bm\theta)$, and a backward policy $P_B(s\mid s';\bm\varphi)$.  For $(s\to s')\in\mathcal E$, define the forward flow and the backward flow as
\[
    F^{\rightarrow}(s\to s'):=F(s;\bm{\theta})P_F(s'|s;\bm{\theta}),\quad
    F^{\leftarrow}(s\to s'):=\begin{cases}
        R(s')P_B(s|s';\bm{\varphi}) & , s'\in\mathcal{X}\\
        F(s';\bm{\theta})P_B(s|s';\bm{\varphi}) & , \text{otherwise}
    \end{cases},
\]
DB trains their equality on every edge using the pointwise loss
\[
    L_{\mathrm{DB}}(s\to s'):=\frac{1}{2}\left(\log\frac{F^{\rightarrow}(s\to s')}{F^{\leftarrow}(s\to s')}\right)^2.
\]

\paragraph{Trajectory Balance GFlowNets\citep{malkin2022trajectory}.}
A TB-GFN uses a positive estimated total flow $\widehat Z(\bm\theta)$, a forward policy $P_F(s'\mid s;\bm\theta)$, and a backward policy $P_B(s\mid s';\bm\varphi)$.  For $\tau=(s_0,s_1,\ldots,s_{n(\tau)}) \in\mathcal T$, define the forward flow and the backward flow as
\[
    F^{\rightarrow}(\tau):=\widehat{Z}(\bm{\theta})\prod_{t=1}^{n(\tau)}P_F(s_t|s_{t-1};\bm{\theta}),\quad
    F^{\leftarrow}(\tau):=R(s_{n(\tau)})\prod_{t=1}^{n(\tau)}P_B(s_{t-1}|s_t;\bm{\varphi}),
\]
TB enforces their equality on every complete trajectory using the pointwise loss
\[
    L_{\mathrm{TB}}(\tau):=\frac{1}{2}\left(\log\frac{F^{\rightarrow}(\tau)}{F^{\leftarrow}(\tau)}\right)^2.
\]

\paragraph{Sampling and Resampling Strategies.}
In practice, the losses and their gradients are estimated from sampled and possibly reweighted states, edges, or trajectories. We denote the resulting distribution by $\mu_{\bm\omega}$, where $\bm\omega$ collects the related parameters. It may depend on the forward policy, the backward policy, or other hyperparameters and can therefore evolve during training.

\paragraph{Backward Policy.} For DB-GFN and TB-GFN, the backward policy $P_B(s\mid s';\bm\varphi)$ can be chosen freely: any positive policy on the incoming edges gives compatible flows and a forward policy by propagating the terminal rewards backward. In practice, $P_B$ is commonly fixed to be uniform or trained jointly using the TB or DB loss. Some recent work \citep{jang2024pessimistic, gritsaev2025optimizing} also studies separate backward-policy criteria designed to improve exploration or exploitation.

\paragraph{Unified Framework\citep{hu2025beyond}.}

For the unified analysis, let $\mathcal O$ denote the relevant training-object space: $\mathcal S\setminus\{s_0\}$ for FM, $\mathcal E$ for DB, and $\mathcal T$ for TB. Write $\mu=\mu_{\bm\omega}$ for a normalized training distribution on $\mathcal O$.
Let $L_2(\mu)$ be the space of real-valued functions on $\mathcal O$ with finite second moment, identified
when they agree $\mu$-almost surely, with
\[
    \langle f,g\rangle_\mu:=\sum_{o\in\mathcal O}\mu(o)f(o)g(o), \qquad
    \|f\|_\mu:=\langle f,f\rangle_\mu^{1/2}.
\]
Denote the \textbf{log-flow imbalance} as
$h_{\bm\theta}(o):=\log\frac{F^{\rightarrow}(o)}{F^{\leftarrow}(o)}$.
The overall objective can be written as
\[
    \mathcal{L}_{\mathcal{O}}(\bm{\theta})
    :=\frac{1}{2}\left\|h_{\bm\theta}\right\|_{\mu}^2
    =\frac{1}{2}\mathbb E_{o\sim\mu}[h_{\bm\theta}(o)^2].
\]

For gradient-based analysis, $h_{\bm\theta}$ and its $\bm\theta$-derivatives are square-integrable under $\mu$, and differentiation may be interchanged with the loss summations in each argument being differentiated; sufficient conditions are given in Appendix~\ref{app:convergence-regularity}. % A $\bm\theta$-derivative holds $\mu$ and any separate $\bm\varphi$ fixed, even when they change between updates.

% When differentiating with respect to $\bm\theta$, we hold $\bm\omega$ and any separate $\bm\varphi$ fixed. We suppress these external arguments until Section~\ref{sec:convergence-rates}. Even if sampling follows the current policy, a balance SGD step differentiates the pointwise loss, not the sampling probabilities; the sampling distribution can change between steps.

\section{From Balance Loss to Distribution Matching}

\label{sec:objectives}

This section connects the balance objectives to the ultimate sampling goal. We first lift the comparison between terminal distributions to complete trajectories and then express the resulting trajectory discrepancy as a weighted sum of the log-flow imbalance $h_{\bm\theta}$.  This yields a bound in terms of the expected training loss and the coverage provided by $\mu$. The proofs are provided in Appendix~\ref{app:proof_4}.

For $\tau=(s_0,s_1,\ldots,s_{n(\tau)})\in\mathcal T$, define the forward and backward trajectory distributions by
\[
    p_F(\tau):=\prod_{t=1}^{n(\tau)} P_F(s_t|s_{t-1}),\qquad
    p_B(\tau):=\frac{R(s_{n(\tau)})}{Z}\prod_{t=1}^{n(\tau)} P_B(s_{t-1}|s_t),
\]
where parameter arguments are suppressed.  The backward trajectory distribution first samples a terminal state from the reward-proportional target and then follows the backward policy to the initial state. To avoid an endless generation process, we assume that both $\mathbb{E}_{p_F}[n(\tau)]$ and $\mathbb{E}_{p_B}[n(\tau)]$ are finite.
For FM and DB on infinite DAGs, we also assume that $\sum_{o\in\tau}|h_{\bm\theta}(o)|$ is integrable under both trajectory laws; see Appendix~\ref{app:proof_4}.
Their terminal marginals are
\[
    p_F^\top(x):=\sum_{\tau:s_{n(\tau)}=x}p_F(\tau),\qquad
    p_R(x):=\frac{R(x)}{Z}=\sum_{\tau:s_{n(\tau)}=x}p_B(\tau),
    \qquad x\in\mathcal X.
\]
The chain rule of KL divergence implies the following lemma.
\begin{lemma}
    \label{lemma:KL_bound}
    Let $p_F$, $p_B$, $p_F^\top$ and $p_R$ be defined as above. We have
    \[
        \mathrm{KL}(p_F^\top\|p_R)\leq \mathrm{KL}(p_F\|p_B),\quad \mathrm{KL}(p_R\|p_F^\top)\leq \mathrm{KL}(p_B\|p_F).
    \]
\end{lemma}

For a training object $o\in\mathcal O$, define its forward and backward occupancy weights by
\[
    \nu_F(o):=\sum_{\tau\in\mathcal T:o\in\tau}p_F(\tau), \qquad
    \nu_B(o):=\sum_{\tau\in\mathcal T:o\in\tau}p_B(\tau),
\]
where, for TB, $o\in\tau$ means $o=\tau$.
% Let
% \[
%     c_\nu:=\sum_{o\in\mathcal O}|\nu_F(o)-\nu_B(o)|.
% \]
% Note that for state and edge objectives, the total mass of each occupancy measure is the expected number of corresponding objects along a trajectory, while for trajectories, the total mass is one. Thus $c_\nu$ is at most the sum of these two total masses.
% When $c_\nu>0$, define the discrepancy distribution as
% \[
%     \nu(o):=\frac{|\nu_F(o)-\nu_B(o)|}{c_\nu}.
% \]
Define the discrepancy distribution as
\[
    \nu(o):=\frac{|\nu_F(o)-\nu_B(o)|}{c_\nu},\quad \text{ where }c_\nu=\sum_{o\in\mathcal O}|\nu_F(o)-\nu_B(o)|.
\]
Note that for state and edge objectives, the total mass of each occupancy measure is the expected number of corresponding objects along a trajectory, while for trajectories, the total mass is one. Thus $c_\nu$ is at most the sum of these two total masses.

To connect trajectory-level KL to balance objectives, we decompose the log-density ratio on trajectories into accumulated local imbalances plus a term that is constant across trajectories.  Symmetrizing the KL divergence cancels this constant and leaves only the imbalances that the training objective controls, yielding the following theorem.

\begin{theorem}
    \label{thm:upper_bound_loss}
    If $c_\nu=0$, then
    $\mathrm{KL}(p_F^\top\|p_R)+\mathrm{KL}(p_R\|p_F^\top)=0$.
    If $c_\nu>0$, then
    \[
        \mathrm{KL}(p_F^\top\|p_R)+\mathrm{KL}(p_R\|p_F^\top)
        % &\leq\sum_{o\in\mathcal O}\left(\nu_F(o)-\nu_B(o)\right)h_{\bm\theta}(o)\\
        \leq c_\nu\sqrt{2\left\|\frac{\nu}{\mu}\right\|_\infty\mathcal{L}_{\mathcal{O}}(\bm{\theta})},
        \quad \text{where } \left\|\frac{\nu}{\mu}\right\|_\infty:=\sup_{o\in\mathcal O}\frac{\nu(o)}{\mu(o)}.
    \]
\end{theorem}

The ratio $\|\nu/\mu\|_\infty$ represents the coverage of the discrepancy distribution by the sampling distribution. We set $\nu(o)/\mu(o)=+\infty$ if $\nu(o)>0=\mu(o)$ and set it to zero if $\nu(o)=\mu(o)=0$.  A finite ratio automatically implies that $\mu$ covers the support of $\nu$; otherwise the bound is valid but vacuous. Although $\nu$ is infeasible in practice, the training objective is effective when $\mu$ provides good coverage. In particular, a bounded coverage ratio and an expected loss of $\delta$ yield an $O(\sqrt{\delta})$ bound on terminal symmetric KL.

Trajectory comparisons also underlie the connection between GFlowNets and variational inference \citep{malkin2023gflownets}. Our coverage factor parallels distribution-mismatch coefficients in RL \citep{agarwal2021theory}, where the relevant measure is the difference between forward and backward occupancies. For approximate balance, \citet{lei2026stable} derive total-variation guarantees from uniform residual bounds. Theorem~\ref{thm:upper_bound_loss} instead controls terminal symmetric KL through the expected squared log-ratio loss under the training measure. A detailed comparison is provided in Appendix~\ref{app:related-gfn-theory}.

% The detailed proofs of Lemma \ref{lemma:KL_bound} and Theorem \ref{thm:upper_bound_loss} are provided in Appendix \ref{app:proof_4}.

\section{From Stationarity to Optimality}
\label{sec:stationary}

% Theorem~\ref{thm:upper_bound_loss} relates distribution mismatch to log-flow imbalance. We now separate the imbalance into its projection onto the local tangent space and an orthogonal residual. The loss gradient controls the projected component through an effective Fisher matrix. Bounding the residual then connects approximate stationarity to terminal-distribution accuracy.

Theorem~\ref{thm:upper_bound_loss} connects distribution error to log-flow imbalance and hence the balance loss. To understand what remains when optimization only reaches a stationary point, we further show that this imbalance can be decomposed into two parts: its projection onto the local tangent space and an orthogonal residual. The loss gradient controls the projected component through an effective Fisher matrix. The orthogonal residual, however, is invisible to first-order optimization, which can be interpreted as a measure of parameterization incompleteness. Therefore, this decomposition offers an approximate stationarity-to-optimality guarantee whose irreducible error is the parameterization-incompleteness term. Below we present the main results, while the proofs are available in Appendix~\ref{app:proof_5}.

For the differentiable imbalance family introduced in Section~\ref{sec:preliminaries}, define the imbalance gradient feature (called the score function below)
\[
    \psi_{\bm\theta}(o):=\nabla_{\bm\theta}h_{\bm\theta}(o).
\]
At a given $\bm\theta$, the score features span the local tangent space
\[
    \mathsf T_{\bm\theta}:=\left\{\bm u^\top\psi_{\bm\theta}(\cdot):\bm u\in\mathbb R^d\right\}.
\]
Thus $h_{\bm\theta}+\mathsf T_{\bm\theta}$ is the local affine model of the imbalance. The distance from this affine space to zero tells us the residual after the best correction in this linearized model.

\begin{definition}[Parameterization Incompleteness]
\label{def:local-residual}
Let $\Pi_{\mathsf T_{\bm\theta}}^\mu$ denote the orthogonal projection in $L_2(\mu)$ onto $\mathsf T_{\bm\theta}$.  Define the tangent residual by
\[
    r_{\mu,\bm\theta}:=\left(I-\Pi_{\mathsf T_{\bm\theta}}^\mu\right)h_{\bm\theta}.
\]
Its parameterization incompleteness under the fixed training measure $\mu$ is
\[
    \epsilon_\mu:=\sup_{\bm\theta\in\Theta}\|r_{\mu,\bm\theta}\|_\mu
    =\sup_{\bm\theta\in\Theta}\min_{\bm u\in\mathbb R^d}\left\|h_{\bm\theta}-\bm u^\top\psi_{\bm\theta}\right\|_\mu
    \in[0,+\infty].
\]
\end{definition}

The residual is orthogonal to every score feature and hence does not appear in the loss gradient. The quantity $\epsilon_\mu$ bounds its norm over the parameter domain. A small value limits the imbalance that first-order optimization cannot detect under $\mu$.

\begin{proposition}[Tabular Zero Residual]
\label{prop:tabular-zero-residual}
On a finite DAG, the standard positive tabular models of FM, DB, and TB satisfy
\[
    h_{\bm\theta}\in\mathsf T_{\bm\theta}\quad\text{for every }\bm\theta\in\Theta.
\]
Consequently, for every sampling measure $\mu$,
\[
    r_{\mu,\bm\theta}=0\quad\text{in }L_2(\mu),\qquad\epsilon_\mu=0.
\]
\end{proposition}

% Proposition~\ref{prop:tabular-zero-residual} provides a concrete zero-residual benchmark. For FM, independent log-flow coordinates make the imbalance Jacobian full row rank. For DB, each edge imbalance can be written as a state-potential difference plus a centered policy advantage. TB appears less direct because there may be far more trajectories than parameters. However, a Bellman recursion again decomposes it into a telescoping value difference and centered advantages. This connects to RL advantage decompositions \citep{sutton1999policy} and potential-based telescoping \citep{ng1999policy}: here, the potential is constructed from the current imbalance to establish its tangent-space representation for both DB and TB. The coordinate constructions are given in Appendix~\ref{app:proof_5}.

\paragraph{Notes on forward-looking parameterizations.}
Forward-looking GFlowNets (FL-GFNs; \citealt{pan2023better}) incorporate a given intermediate energy into the balance equations. With the convention $R(x)=e^{-E(x)}$, their state-flow parameterization can be written as
\[
    F(s;\bm\theta)=e^{-E(s)}\widetilde F(s;\bm\theta),
    \qquad \widetilde F(x)=1\quad(x\in\mathcal X).
\]
When parameterized independently by the log-flows, FL-GFNs differ only by a fixed translation, so Proposition~\ref{prop:tabular-zero-residual} still applies.

With Proposition~\ref{prop:tabular-zero-residual}, $\epsilon_\mu$ can be viewed as a reasonable measure of parameterization incompleteness. The remainder of the section concerns a general parameterized family. We require two regularity conditions to control the tangent-space component.

\begin{assumption}[Bounded Score Function]
\label{assump:score}
There exists $M>0$ such that, for the discrepancy distribution $\nu$ from Theorem~\ref{thm:upper_bound_loss} and every $\bm\theta\in\Theta$ under consideration with $c_\nu>0$,
\[
    \mathbb E_{o\sim\nu}\left[\|\psi_{\bm\theta}(o)\|_2\right]\leq M,
\]
where $\nu$ is induced at the same $\bm\theta$.
\end{assumption}

\begin{assumption}[Effective Fisher Conditioning]
\label{assump:fisher}
Define the Fisher information as
\[
    \mathcal I_\mu(\bm\theta):=\mathbb E_{o\sim\mu}[\psi_{\bm\theta}(o)\psi_{\bm\theta}(o)^\top].
\]
There exists $\lambda>0$ such that, for every $\bm\theta\in\Theta$ under consideration,
\[
    \left\|\mathcal I_\mu(\bm\theta)^\dagger\right\|_{2\to2}\leq\frac{1}{\lambda},
\]
where $\mathcal I_\mu(\bm\theta)^\dagger$ is the Moore--Penrose pseudo-inverse.
\end{assumption}

Assumption~\ref{assump:score} controls how large a tangent direction can be on the objects emphasized by $\nu$.  Assumption~\ref{assump:fisher} requires the nonzero eigenvalues of the Fisher matrix to be uniformly separated from zero.  Unlike strict positive definiteness, it permits null directions caused by redundant coordinates and therefore accommodates over-parameterized models.  Only the effective parameter subspace matters because null directions have zero first-order effect on the imbalance in $L_2(\mu)$. They need not be invisible outside the support of $\mu$, which is why coverage remains a separate condition.

\begin{lemma}[Tangent Projection from the Gradient]
\label{lemma:local-coefficient}
Let
\[
    \bm u_\mu^*(\bm\theta):=\mathcal I_\mu(\bm\theta)^\dagger\nabla_{\bm\theta}\mathcal L_{\mathcal O}(\bm\theta).
\]
Under Assumption~\ref{assump:fisher},
\[
    \Pi_{\mathsf T_{\bm\theta}}^\mu h_{\bm\theta}={\bm u_\mu^*(\bm\theta)}^\top\psi_{\bm\theta},\qquad
    \|\bm u_\mu^*(\bm\theta)\|_2\leq\frac{1}{\lambda}\|\nabla_{\bm\theta}\mathcal L_{\mathcal O}(\bm\theta)\|_2.
\]
\end{lemma}

This projection parallels feature regression in compatible function approximation and natural policy gradient \citep{sutton1999policy,kakade2001natural}. Here the regression target is the imbalance $h_{\bm\theta}$ and the features are its derivatives $\nabla_{\bm\theta}h_{\bm\theta}$. Accordingly, the effective Fisher is the Gram matrix of imbalance features, and the residual measures local tangent approximation error. Refer to Appendix~\ref{app:related-techniques} for further discussions. Together with conditioning and coverage, this geometry links balance-loss stationarity to terminal sampling accuracy in the following theorem.

\begin{theorem}
\label{thm:stationarity_to_optimality}
Suppose Assumptions~\ref{assump:score} and~\ref{assump:fisher} hold. For every $\bm\theta\in\Theta$, if $c_\nu>0$, then
\begin{align*}
\mathrm{KL}(p_F^\top\|p_R)+\mathrm{KL}(p_R\|p_F^\top)
\leq c_\nu\left(
\sqrt{\left\|\frac{\nu}{\mu}\right\|_\infty}
\,\epsilon_\mu
+\frac{M}{\lambda}
\left\|\nabla_{\bm\theta}
\mathcal L_{\mathcal O}(\bm\theta)\right\|_2
\right).
\end{align*}
If $c_\nu=0$, the symmetric KL divergence on the left-hand side is zero.
The coverage term is understood as $+\infty$ whenever $\|\nu/\mu\|_\infty=+\infty$.
\end{theorem}

This theorem bounds distribution mismatch by the gradient norm and a residual term. The factor $M/\lambda$ describes the conditioning of the visible component, while the coverage ratio controls transfer from training objects to occupancy discrepancies. With finite coverage, the zero-residual tabular case has zero terminal symmetric KL at every stationary point satisfying the assumptions.

\section{Convergence Rates of Stochastic GFlowNet Training}
\label{sec:convergence-rates}

Section~\ref{sec:stationary} explains when a small gradient implies a small imbalance. We now ask whether stochastic training reaches such a regime, and how quickly. We first fix the sampling distribution, so that tangent geometry directly turns SGD descent into loss contraction. We then allow the sampling distribution to change, and finally allow a separate backward-policy update. The stochastic-optimization, moving-objective, and actor--critic background for is summarized in Appendix~\ref{app:related-techniques}; detailed proofs and regularity checks appear in Appendix~\ref{app:convergence-proofs}. % Each extension introduces a drift that must be controlled relative to the descent.

\subsection{Fixed sampling: from geometry to convergence}

Recall that $\bm\theta$ contains all parameters trained by the balance loss, including the backward parameters for joint balance training. Make the external arguments of the same expected loss explicit:
\[
    \mathcal L_{\mathcal O}(\bm\theta,\bm\omega;\bm\varphi)
    :=\frac12\mathbb E_{o\sim\mu_{\bm\omega}}
       [h_{\bm\theta,\bm\varphi}(o)^2].
    % \label{eq:dynamic-balance-objective}
\]
The separate argument $\bm\varphi$ is omitted for joint balance training and for FM; in those cases write $h_{\bm\theta}$ in place of $h_{\bm\theta,\bm\varphi}$. For now, $\bm\omega$ and any separate $\bm\varphi$ are fixed. Joint balance training simply enlarges the optimization block: its backward-policy update is part of the descent on the same objective.

Let $\mathcal F_t$ contain the history before iteration $t$ and write $\mathbb E_t=\mathbb E[\,\cdot\mid\mathcal F_t]$. All current parameters are $\mathcal F_t$-measurable, and $\mathcal L_t:=\mathcal L_{\mathcal O}(\bm\theta_t,\bm\omega_t;\bm\varphi_t)$ denotes their current loss. The balance update satisfies
\begin{equation}
    \bm\theta_{t+1}=\bm\theta_t-\eta_t\bm g_t,\qquad
    \mathbb E_t\bm g_t=\nabla_{\bm\theta}\mathcal L_t,\qquad
    \mathbb E_t\|\bm g_t\|_2^2
    \leq\|\nabla_{\bm\theta}\mathcal L_t\|_2^2+\sigma^2.
    \label{eq:balance-sgd}
\end{equation}
This derivative holds the sampling distribution and any separate backward parameters fixed, including when they change between iterations below.
We assume $\mathbb E\mathcal L_0<\infty$, finite second moments of all
updates, and finite score second moments. All constants are deterministic and independent of $t$.

%Two ingredients give the baseline rate. Smoothness controls the error of a stochastic descent step, while tangent geometry ensures that the gradient detects imbalance beyond the tangent-residual bound.
\begin{assumption}[Smooth stochastic objective]
\label{assump:sgd-smoothness}
The oracle satisfies~\eqref{eq:balance-sgd}. For some $L>0$, the loss has an
$L$-Lipschitz gradient jointly in the arguments that change, on a convex
admissible region containing every actual update segment.
\end{assumption}
\begin{assumption}[Uniform tangent geometry]
\label{assump:uniform-tangent-geometry}
For the residual and Fisher matrix defined in Section~\ref{sec:stationary},
there exist $\bar\epsilon\geq0$ and $\lambda>0$ such that, almost surely
at every iterate,
\[
    \|r_{\mu_{\bm\omega_t},\bm\theta_t}\|_{\mu_{\bm\omega_t}}
    \leq\bar\epsilon,\qquad
    \|\mathcal I_{\mu_{\bm\omega_t}}(\bm\theta_t)^\dagger\|_{2\to2}
    \leq\frac1\lambda.
\]
Both quantities use $\psi_{\bm\theta}=\nabla_{\bm\theta}h$, with any separate backward parameters fixed at their current values.
\end{assumption}

Here $\bar\epsilon$ bounds the residual along training, including changes in the external arguments. The projection identity gives an approximate Polyak--\L{}ojasiewicz (PL) inequality:
\begin{lemma}[Approximate PL geometry]
\label{lemma:approximate-PL}
Under Assumption~\ref{assump:uniform-tangent-geometry}, at every iterate,
\begin{equation}
    2\mathcal L_t\leq\bar\epsilon^2+
       \frac1\lambda\|\nabla_{\bm\theta}\mathcal L_t\|_2^2.
    \label{eq:approximate-PL}
\end{equation}
\end{lemma}
Thus the gradient detects loss above the residual level. Note that for a sufficiently small step-size, smoothness gives expected descent of order $\eta_t\|\nabla_{\bm\theta}\mathcal L_t\|_2^2$, up to an $O(\eta_t^2)$ noise term. The lemma therefore converts this into contraction of $\mathcal L_t$ outside an $O(\bar\epsilon^2)$ residual neighborhood.
\begin{theorem}[SGD with fixed sampling]
\label{thm:sgd-fixed-sampling}
Suppose Assumptions~\ref{assump:sgd-smoothness}
and~\ref{assump:uniform-tangent-geometry} hold. If $\mu$ is fixed and
$P_B$ is fixed or jointly balance-trained, there exist $a,t_0,C>0$ such
that $\eta_t=a/(t+t_0)$ yields
\begin{equation}
    \mathbb E\mathcal L_T
    \leq C\bar\epsilon^2+\frac{C}{T+t_0}.
    \label{eq:sgd-loss-rate}
\end{equation}
\end{theorem}

% To recover the sampling goal, we assume the uniform bound of occupancy mass $c_{\nu,t}\leq\bar{B}$, and the uniform coverage bound $\|\nu_t/\mu_{\bm\omega_t}\|_\infty\leq\bar C_\nu$ whenever $c_{\nu,t}>0$. Theorem~\ref{thm:upper_bound_loss} and Jensen's inequality then give
% \begin{equation}
%     \mathbb E\left[\mathrm{KL}(p_{F,T}^\top\|p_R)
%        +\mathrm{KL}(p_R\|p_{F,T}^\top)\right]
%     \leq C\bar\epsilon+\frac{C}{\sqrt{T+t_0}}.
%     \label{eq:sgd-terminal-rate}
% \end{equation}
% On finite DAGs, standard positive tabular models have $\bar\epsilon=0$ by Proposition~\ref{prop:tabular-zero-residual}. Under the remaining oracle, smoothness, conditioning, and coverage conditions, their expected loss and terminal symmetric KL therefore decay as $O(T^{-1})$ and $O(T^{-1/2})$.

To convert loss convergence into terminal accuracy, we require:
\begin{assumption}
\label{assump:terminal-control}
Assume $0<Z<\infty$. Almost surely at every iterate, policies are normalized and positive on allowed edges, and flow quantities in $h$ are finite and positive. Forward trajectories reach a terminal state and backward trajectories reach $s_0$ almost surely. For FM and DB, trajectory sums of $|h|$ are integrable under both $p_{F,t}$ and $p_{B,t}$.
For $\bar B,\bar C_\nu<\infty$, uniformly in $t$,
\[
    c_{\nu,t}\leq\bar B,\qquad
    \left\|\frac{\nu_t}{\mu_{\bm\omega_t}}\right\|_\infty
       \leq\bar C_\nu\quad\text{whenever }c_{\nu,t}>0.
\]
\end{assumption}
For TB, $c_{\nu,t}\leq2$ automatically; finite DAGs also ensure the FM/DB mass bound and integrability.

\begin{corollary}
\label{cor:sgd-terminal-rate}
Under Assumption~\ref{assump:terminal-control}, the loss bound~(\ref{eq:sgd-loss-rate}) implies, for some $C>0$,
\begin{equation}
    \mathbb E\left[\mathrm{KL}(p_{F,T}^\top\|p_R)
       +\mathrm{KL}(p_R\|p_{F,T}^\top)\right]
    \leq C\bar\epsilon+\frac{C}{\sqrt{T+t_0}}.
    \label{eq:sgd-terminal-rate}
\end{equation}
\end{corollary}
This follows from Theorem~\ref{thm:upper_bound_loss} and Theorem~\ref{thm:sgd-fixed-sampling}. See Appendix~\ref{app:convergence-fixed} for a detailed discussion.

% On finite DAGs, standard positive tabular models have $\bar\epsilon=0$ by Proposition~\ref{prop:tabular-zero-residual}. Under the remaining oracle, smoothness, conditioning, and coverage conditions, their expected loss and terminal symmetric KL therefore decay as $O(T^{-1})$ and $O(T^{-1/2})$.

For a different loss function, \citet{yu2025secrets} establishes best-iterate gradient convergence for FM and DB of $O(T^{-1/2})$ and $O(T^{-1/3})$, respectively, while our result provides $O(T^{-1/2})$ last-iterate global convergence for FM, DB and TB. A detailed comparison is provided in Appendix~\ref{app:related-gfn-theory}.

\subsection{Changing sampling: retaining net descent}
\label{sec:changing-sampling}

When training samples follow the evolving policies, changing $\mu$ changes the weights in the objective. Even a successful balance step can therefore be partly offset by the next sampling update. Write $\Delta_t^\omega:=\bm\omega_{t+1}-\bm\omega_t$. We make the following assumptions on the evolving sampling distribution.
\begin{assumption}
    \label{assump:controlled-omega-gradient}
    For $C_\omega>0$, almost surely at every iterate,  $$\|\nabla_{\bm\omega}\mathcal L_t\|_2\leq C_\omega\sqrt{2\mathcal L_t}.$$
\end{assumption}

\begin{assumption}
\label{assump:controlled-sampling} For constants $K_\theta,K_2\geq0$, almost surely at every iterate,
\begin{align*}
    \|\mathbb E_t\Delta_t^\omega\|_2
       &\leq K_\theta\eta_t\|\nabla_{\bm\theta}\mathcal L_t\|_2,\\
    \mathbb E_t\|\Delta_t^\omega\|_2^2
       &\leq K_2\eta_t^2
          (\|\nabla_{\bm\theta}\mathcal L_t\|_2^2+\sigma^2).
\end{align*}
For fixed sampling, these conditions are vacuous and the constants are zero.
\end{assumption}

Assumption~\ref{assump:controlled-omega-gradient} follows from bounded $h$ and a bounded sampling score $\|\nabla_{\bm\omega}\log\mu_{\bm\omega}(o)\|_2$. Assumption~\ref{assump:controlled-sampling} includes direct on-policy parameter copying, but modifications such as exploration mixtures, tempering, or delayed updates require their own sensitivity and update bounds. Appendix~\ref{app:convergence-regularity} gives sufficient conditions for these two assumptions. Together with the PL inequality, these bounds leave net descent when
\[
    \rho:=\frac{C_\omega K_\theta}{\sqrt\lambda}<1.
    % \label{eq:sampling-drift-coefficient}
\]

\begin{theorem}[SGD with controlled sampling]
\label{thm:sgd-controlled-sampling}
Suppose Assumptions~\ref{assump:sgd-smoothness}, \ref{assump:uniform-tangent-geometry}, \ref{assump:controlled-omega-gradient}, and~\ref{assump:controlled-sampling} hold and $\rho<1$. If there is no separate backward update (FM, or TB/DB with fixed or jointly balance-trained $P_B$), then, for suitable $a,t_0>0$, $\eta_t=a/(t+t_0)$ gives the loss bound~(\ref{eq:sgd-loss-rate}) and, under Assumption~\ref{assump:terminal-control}, the terminal bound~(\ref{eq:sgd-terminal-rate}).
\end{theorem}

% In short, controlled sampling changes the constants, but preserves the baseline orders. The restriction $\rho<1$ is sufficient, and is automatic for fixed sampling.

\paragraph{Practical interpretation of the drift condition.}
Theorem~\ref{thm:sgd-controlled-sampling} identifies $\rho<1$ as a sufficient condition for retaining the baseline rates under changing sampling. To obtain a smaller $\rho$, we can apply fixed exploration mixtures, tempering and experience replay to reduce $C_\omega$, or incorporate damped updates or target updates to reduce $K_\theta$. Refer to Appendix~\ref{app:sampling-stabilization} for a detailed discussion.

\subsection{Separate backward updates: controlling the extra drift}
\label{sec:changing-pb}

For TB and DB, updating $\bm\varphi$ by another criterion changes the balance equations and need not decrease the balance loss. Define
\[
    z_t:=\bigl(\mathbb E_t\|\bm\varphi_{t+1}-\bm\varphi_t\|_2^2\bigr)^{1/2},
    \qquad d_t:=(\mathbb E z_t^2)^{1/2}.
\]
Similar to the sampling parameters, we make the following assumptions.

\begin{assumption}
    \label{assump:controlled-varphi-gradient}
    For $C_\varphi>0$, almost surely at every iterate,  $$\|\nabla_{\bm\varphi}\mathcal L_t\|_2\leq C_\varphi\sqrt{2\mathcal L_t}.$$
\end{assumption}

\begin{assumption}
\label{assump:controlled-sampling-extension} For constants $K_\theta,K_\varphi,K_2\geq0$, almost surely at every iterate,
\begin{align*}
    \|\mathbb E_t\Delta_t^\omega\|_2
       &\leq K_\theta\eta_t\|\nabla_{\bm\theta}\mathcal L_t\|_2+K_\varphi z_t,\\
    \mathbb E_t\|\Delta_t^\omega\|_2^2
       &\leq K_2\left[\eta_t^2
          (\|\nabla_{\bm\theta}\mathcal L_t\|_2^2+\sigma^2)+z_t^2\right].
\end{align*}
For fixed sampling, these conditions are vacuous and the constants are zero.
\end{assumption}

% For TB and DB, a backward policy trained by another criterion changes the
% balance equations themselves, although their terminal target remains
% $p_R$. Here $\bm\varphi$ is external to $\bm\theta$. Unlike joint balance training, this update
% need not decrease $\mathcal L_{\mathcal O}$, so its effect must be bounded separately.
% Define its conditional and unconditional root-mean-square magnitudes by
% \[
%     z_t:=\bigl(\mathbb E_t\|\bm\varphi_{t+1}-\bm\varphi_t\|_2^2\bigr)^{1/2},
%     \qquad d_t:=(\mathbb E z_t^2)^{1/2}.
% \]
% For this changing block, assume
% \begin{equation}
%     \|\nabla_{\bm\varphi}\mathcal L_t\|_2
%        \leq C_\varphi\sqrt{2\mathcal L_t}.
%     \label{eq:backward-sensitivity}
% \end{equation}
% Bounded backward-policy scores imply this condition on a finite DAG.
% To allow sampling to respond to both parameter blocks, retain the sampling
% sensitivity bound and extend its two update bounds to
% \begin{equation}
% \begin{aligned}
%     \|\mathbb E_t\Delta_t^\omega\|_2
%        &\leq K_\theta\eta_t\|\nabla_{\bm\theta}\mathcal L_t\|_2+K_\varphi z_t,\\
%     \mathbb E_t\|\Delta_t^\omega\|_2^2
%        &\leq K_2\left[\eta_t^2
%           (\|\nabla_{\bm\theta}\mathcal L_t\|_2^2+\sigma^2)+z_t^2\right].
% \end{aligned}
% \label{eq:PB-controlled-sampling}
% \end{equation}
% Here $C_\varphi,K_\varphi$ are finite nonnegative constants; sampling
% constants are again zero when $\mu$ is fixed.

\begin{theorem}[SGD with separate backward updates]
\label{thm:sgd-general-PB}
Suppose Assumptions~\ref{assump:sgd-smoothness},\ref{assump:uniform-tangent-geometry}, \ref{assump:controlled-omega-gradient}, \ref{assump:controlled-varphi-gradient}, and~\ref{assump:controlled-sampling-extension} hold, with $\rho<1$. For suitable $a,t_0,c,C>0$, $\eta_t=a/(t+t_0)$ yields
\begin{equation}
    \mathbb E_t\mathcal L_{t+1}
    \leq(1-c\eta_t)\mathcal L_t
       +C\eta_t\bar\epsilon^2+C\eta_t^2
       +C\frac{z_t^2}{\eta_t}.
    \label{eq:general-PB-recursion-1}
\end{equation}
\begin{enumerate}
    \item If $d_t=o(\eta_t)$, then
    $\limsup\limits_{t\to\infty}\mathbb E\mathcal L_t\leq C\bar\epsilon^2$.
    In particular, $d_t=O(t^{-b})$ with any $b>1$ suffices.
    \vspace{-8pt}
    \item If $d_t=O(\eta_t^{3/2})$, equivalently $d_t=O(t^{-3/2})$,
    then $\mathbb E\mathcal L_T=O(\bar\epsilon^2+T^{-1})$.
    \item If $d_t=O(t^{-b})$ with $1<b<3/2$, then
    $\mathbb E\mathcal L_T=O(\bar\epsilon^2+T^{-2(b-1)})$.
\end{enumerate}
Under Assumption~\ref{assump:terminal-control}, case (i) gives an $O(\bar\epsilon)$ neighborhood for expected terminal symmetric KL. In cases (ii) and (iii), the expected terminal symmetric KL is bounded, respectively, by $O(\bar\epsilon+T^{-1/2})$ and $O(\bar\epsilon+T^{-(b-1)})$.
\end{theorem}

The three additive terms in~\eqref{eq:general-PB-recursion-1} separate representation error, SGD noise, and backward-policy drift. After taking expectation, the last is $Cd_t^2/\eta_t$. For the expected loss to approach the residual neighborhood it suffices that this be $o(\eta_t)$; to retain the baseline rate it suffices that it be $O(\eta_t^2)$, the same order as SGD noise. 

\section{Conclusion}

In this paper, we develop a unified theory that connects balance loss, approximate stationarity, and sampling accuracy for FM, DB, and TB. We bound terminal symmetric KL by the square root of the expected balance loss and characterize stationary-point quality through a tangent-space residual that vanishes for standard positive tabular models. Under suitable regularity, coverage, and conditioning assumptions, we establish finite-time convergence guarantees and quantify how sampling and backward-policy drift must be controlled to preserve them. Our findings clarify the conditions under which optimization progress leads to accurate reward-proportional sampling, marking a significant advancement in understanding the training of GFlowNets. This insight is valuable for further algorithmic development and real-world applications. Moreover, our analysis builds upon and enhances existing function-approximation and distribution-shift techniques from RL and optimization, indicating their potential applicability to a broader range of fields.

However, the guarantees outlined in this paper depend on adequate coverage, well-conditioned effective tangent geometry, and controlled drift. These factors require further verification in practical implementations, particularly when neural networks are involved. We recognize these limitations and view them as important directions for future research, despite the challenges they present.

\clearpage
\FloatBarrier
\begingroup
\small
\urlstyle{same}
\bibliographystyle{dilab_ref}
\bibliography{iclr2027_conference}
\endgroup

\makeappendixtoc
\appendix
\appendix

\section{Related Work and Technical Perspective}
\label{app:related-work}

We review the theoretical development of GFlowNets, emphasizing results on sampling accuracy and learning, and then explain how our analysis connects to techniques in reinforcement learning and optimization.

\subsection{Theoretical Progress in GFlowNets}
\label{app:related-gfn-theory}

\paragraph{Foundations and connections to other theories.}
The foundations of GFlowNets characterize reward-proportional sampling through consistent flows on a construction graph \citep{bengio2021flow,bengio2023gflownet}. Flow Matching (FM) and Detailed Balance (DB) impose local consistency, while Trajectory Balance (TB) enforces consistency along complete generation paths \citep{malkin2022trajectory}. Under the corresponding support and reward-matching conditions, exact balance yields the target terminal distribution. Sub-Trajectory Balance (SubTB) interpolates between edge and complete-trajectory constraints, providing intermediate scopes for credit assignment \citep{madan2023learning}. Appendix~\ref{app:subtb} extends our analysis to SubTB under its stated conditions. These results establish the correctness of balanced solutions; understanding an imperfectly trained, parameterized sampler requires additional analysis.

Connections to other inference and learning frameworks clarify these objectives. \citet{malkin2023gflownets} establish expected-gradient identities between balance objectives and variational inference under specified trajectory distributions and choices of the differentiated parameter block. \citet{tiapkin2024generative} reformulate GFlowNet learning as entropy-regularized RL using a particular reward and regularizer structure. More recently, \citet{deleu2025relative} identify an equivalence between Relative Trajectory Balance and the KL-regularized RL method Trust-PCL. Loss-design theory further connects regression penalties to distributional objectives: \citet{hu2025beyond} relate expected regression gradients to gradients of divergences between flows on minimal cuts under an appropriate sampling structure. \citet{fawkes2026trajectory} develop an $f$-trajectory-balance family with corresponding on-policy $f$-divergence gradients and target-preserving off-policy minima under their conditions. Such identities illuminate optimization signals and loss design, but require further arguments to characterize parameterized stationary points or finite-time sampling accuracy.

\paragraph{Imbalance sensitivity and representational limits.}
\citet{silva2025when} investigate how balance violations affect sampling accuracy. Their Theorem~1 bounds terminal total variation (TV) for a specified perturbation of an initially balanced network, obtained by adding and propagating flow. The bounds reveal dependence on the flow scale and target mass of downstream terminal states. This is a sensitivity result for the stated perturbation construction, rather than a general bound on arbitrary expected training loss. They also characterize expressiveness limitations of GFlowNets using 1-WL graph neural networks, propose more expressive look-ahead policies, and introduce flow consistency in subgraphs (FCS) as a distributional diagnostic with theoretical connections to TV. Their representational results concern which distributions a model class can express. Our tangent residual instead measures imbalance unexplained by local parameter derivatives at the current model.

\paragraph{Generalization, coverage, and statistical guarantees.}
\citet{krichel2024on} formalize generalization and reward-perturbation stability and empirically study generalization to longer trajectories. Their stability proposition assumes globally minimized TB, a fixed uniform backward policy, and reward functions with equal total mass; the length-generalization experiments do not establish an unrestricted extrapolation guarantee. \citet{silva2025generalization} develop data-dependent PAC-Bayesian bounds using bounded FCS risk, together with oracle bounds for terminal KL and DB-based risk under their sampling and regularity assumptions. Their Proposition~4.2 also bounds terminal TV through an expected squared trajectory log-density ratio and a distribution-mismatch factor involving $\chi^2$ divergence. Its expectation samples terminal states from a chosen distribution and then backward trajectories; the ratio uses normalized target probabilities. Our training measure can instead weight states, edges, or complete trajectories, with normalization terms eliminated by symmetrization. Thus, coverage-dependent links between average discrepancy and sampling error already exist. Their analysis motivates Subgraph Asynchronous Learning, which divides training across smaller subnetworks. Our contribution is a common occupancy-discrepancy analysis for FM, DB, and TB, and its connection to stationary-point quality and stochastic optimization. The population bounds here complement statistical generalization guarantees: coverage alone does not justify substituting empirical loss for population loss.

\paragraph{Comparison with \citet{yu2025secrets}.}
Yu studies four aspects of learning behavior---convergence, sample complexity, implicit regularization, and robustness---using the following objectives\footnote{There is a typo in their definition of TB loss, where the backward probabilities are missing.} (interpreted under our notations):
\[\hat{L}_{\mathrm{FM}}(s)=\left(F^{\rightarrow}(s)-F^{\leftarrow}(s)\right)^2,\quad \hat{L}_{\mathrm{DB}}(s)=\left(\frac{F^{\rightarrow}(s)}{F^{\leftarrow}(s)}-1\right)^2, \quad \hat{L}_{\mathrm{TB}}(s)=\left(\frac{F^{\rightarrow}(s)}{F^{\leftarrow}(s)}-1\right)^2.\]

These definitions differ from the squared log-ratio objectives used in practical implementations \citep{hernandez-garcia2024,torchgfn2026software} and in this paper. Different losses have different weighting and gradients, so their convergence statements are not interchangeable.

Under tabular parameterization, fixed sampling and fixed backward policies, Yu's Theorems~2 and~3 give $O(T^{-1/2})$ and $O(T^{-1/3})$ rates for the expected minimum squared gradient norm of FM and DB over all $T$ iterates. In contrast, our Theorem~\ref{thm:sgd-fixed-sampling} give $O(T^{-1/2})$ rate for the expected terminal symmetric KL of the $T$'s iterate of FM, DB and TB.

\paragraph{Comparison with \citet{lei2026stable}.}
Lei et al. study loss sensitivity, showing that small terminal TV need not prevent large balance loss, and develop guarantees in the converse direction. Their Theorem~3.5 bounds TV by $1-e^{-2c}$ when every squared TB residual is at most $c^2$; the corresponding FM/DB bound is $1-e^{-2Lc}$ for maximum trajectory length $L$. Their Theorem~3.6 provides a probabilistic certificate using independent forward trajectories and backward trajectories initialized from the target distribution. Its global form requires target-terminal samples; restricted-target sampling yields a subgraph certificate. They also propose adaptive reference flows to stabilize training, analyzing the associated fidelity trade-off.

Relative to Lei et al., our distributional statement uses expected squared log-ratio loss under a general, possibly off-policy training measure and controls terminal symmetric KL through occupancy mass and coverage. It does not require their uniform pointwise residual condition, but it can be vacuous under poor coverage and is a population inequality rather than a finite-sample certificate. The two analyses therefore control different error metrics under different hypotheses; neither gives an unconditional tightness ordering.

\paragraph{Extensions of the underlying environment.}
Other theoretical work broadens the modeling setting. \citet{lahlou2023theory} formulate generalized GFlowNets using measurable spaces, flow measures, and Markov kernels, extending balance conditions and sampling correctness to continuous and hybrid spaces under assumptions including finite absorption. \citet{morozov2025revisiting} revisit discrete non-acyclic environments, clarifying flow definitions, the role of backward policies, and the relationship between total flow and expected trajectory length. They also analyze fixed-backward training and stability issues when the backward policy is learned. These extensions address structural questions beyond the discrete DAG setting used here; transferring our approximate-balance and optimization analysis to them requires revisiting the occupancy and trajectory assumptions.

\subsection{Connections to Reinforcement Learning and Optimization}
\label{app:related-techniques}

\paragraph{Compatible approximation, natural gradients, and approximate PL geometry.}
Compatible function approximation regresses advantages onto policy-score features, connecting least-squares coefficients to natural policy gradients (NPG) \citep{sutton1999policy,kakade2001natural}. For general policy classes, transferred approximation error evaluates this regression under a comparator policy's occupancy, making distribution shift part of the optimality analysis \citep{agarwal2021theory}. Under Fisher nondegeneracy and bounded transferred error, \citet{liu2020improved} analyze NPG through compatible regression and a comparator-policy KL potential. For momentum policy gradient, \citet{ding2022global} derive a bound on the optimality gap by an approximation floor plus a multiple of the gradient norm. Related analyses develop this route for stochastic policy gradients \citep{fatkhullin2023stochastic} and NPG with Hessian-aided momentum variance reduction \citep{feng2024global}.

In the context of GFlowNets, the regression target is log-flow imbalance and the features are its parameter derivatives. The residual is local to the tangent space and measured under the training distribution, unlike transferred advantage error under a comparator occupancy. The feature Gram matrix is the Gauss--Newton matrix of the squared-imbalance loss with the training measure fixed, and need not be a policy Fisher. \citet{raykov2026information} use the forward trajectory Fisher to precondition GFlowNet balance-loss gradients. This Fisher coincides with the imbalance Gram matrix for the forward-only TB block under on-policy weighting with backward policy and source flow held fixed; their decomposition and approximation guarantees concern update geometry rather than approximation-dependent gradient domination. A lower bound on the nonzero Gram eigenvalues turns the projection identity into an approximate PL inequality: squared gradient norm controls loss above the tangent-residual floor. This also connects to nonlinear least-squares and tangent-kernel analyses \citep{karimi2016linear,jacot2018neural,liu2022loss}. 

\paragraph{Stochastic-gradient baselines and objective drift.}
The fixed-sampling part of Section~\ref{sec:convergence-rates} belongs to the classical stochastic-approximation lineage initiated by Robbins and Monro~\citep{robbins1951stochastic} and developed for recursive algorithms and stochastic optimization in later treatments~\citep{bottou2018optimization}. Conditional unbiasedness and a second-moment bound for the balance-gradient oracle, together with smoothness, give the usual expected-descent recursion. In our setting, the effective Fisher lower bound and the tangent residual turn this recursion into an approximate Polyak--Lojasiewicz inequality: the loss contracts until it reaches the parameterization residual floor, with an $O(T^{-1})$ loss rate under harmonic stepsizes. The terminal-distribution bound then takes a square root, giving the $O(T^{-1/2})$ rate in Corollary~\ref{cor:sgd-terminal-rate}. This is a population result for the stated training measure, rather than a generic claim that every nonconvex SGD trajectory converges globally.

When the sampling law changes with the iterate, or the backward policy is updated seperately, the objective seen by the SGD oracle thus evolves during training. Two-time-scale stochastic approximation provides the classical limiting-ODE and finite-time perspectives for coupled updates~\citep{borkar1997stochastic,konda2004convergence,kaledin2020finite}. The controlled-gradient and controlled-sampling assumptions in Section~\ref{sec:changing-sampling} and Section~\ref{sec:changing-pb} make the first- and second-moment effects of this drift explicit. This interpretation is related to small-gain analyses of single-timescale actor--critic methods~\citep{olshevsky2023a,chen2023finite} and to finite-time analyses with an evolving reward or objective~\citep{hu2026finitetime}, but the constants here are defined for the GFlowNet sampling measure and its imbalance features.

\section{Proofs for Section \ref{sec:objectives}}

\label{app:proof_4}

\paragraph{Proof of Lemma \ref{lemma:KL_bound}.}

The chain rule for relative entropy, applied to the terminal-state map
$\tau\mapsto s_{n(\tau)}$, gives
\begin{align*}
\mathrm{KL}(p_F\|p_B)
&=\mathrm{KL}(p_F^\top\|p_R)
+\sum_{x\in\mathcal X}p_F^\top(x)
\mathrm{KL}\bigl(p_F(\cdot\mid s_{n(\tau)}=x)\|p_B(\cdot\mid s_{n(\tau)}=x)\bigr)\\
&\geq\mathrm{KL}(p_F^\top\|p_R),\\
\mathrm{KL}(p_B\|p_F)
&=\mathrm{KL}(p_R\|p_F^\top)
+\sum_{x\in\mathcal X}p_R(x)
\mathrm{KL}\bigl(p_B(\cdot\mid s_{n(\tau)}=x)\|p_F(\cdot\mid s_{n(\tau)}=x)\bigr)\\
&\geq\mathrm{KL}(p_R\|p_F^\top).
\end{align*}
\hfill $\square$\par

\paragraph{Proof of Theorem \ref{thm:upper_bound_loss}.}

For FM and DB, the integrability condition in Section~\ref{sec:objectives}
and Tonelli's theorem give
\[
    \sum_{o\in\mathcal O}(\nu_F(o)+\nu_B(o))|h_{\bm\theta}(o)|
    =\mathbb E_{p_F}\sum_{o\in\tau}|h_{\bm\theta}(o)|
      +\mathbb E_{p_B}\sum_{o\in\tau}|h_{\bm\theta}(o)|<\infty.
\]
This condition is automatic on a finite DAG. It justifies the signed
sum interchange below by Fubini's theorem, since
$|p_F(\tau)-p_B(\tau)|\leq p_F(\tau)+p_B(\tau)$.
For TB, objects are complete trajectories, so no such interchange is
needed. If $c_\nu=0$, then $p_F=p_B$ and the TB claim is immediate.
Otherwise, it suffices to consider finite coverage ratio and finite loss,
since the claimed bound is vacuous in the remaining cases. Then
\[
    \sum_{\tau\in\mathcal T}|p_F(\tau)-p_B(\tau)|\,
       |h_{\bm\theta}(\tau)|
    =c_\nu\mathbb E_\nu|h_{\bm\theta}|
    \leq c_\nu\sqrt{2\left\|\frac\nu\mu\right\|_\infty
       \mathcal L_{\mathcal O}(\bm\theta)}<\infty,
\]
by Cauchy--Schwarz and change of measure. Thus the signed sums below
are absolutely convergent in every nonvacuous case.

According to Lemma \ref{lemma:KL_bound}, we have
\begin{align*}
    \mathrm{KL}(p_F^\top\|p_R)+\mathrm{KL}(p_R\|p_F^\top)
    &\leq\mathrm{KL}(p_F\|p_B)+\mathrm{KL}(p_B\|p_F)\\
    &=\sum_{\tau\in\mathcal T}\left(p_F(\tau)-p_B(\tau)\right)
    \log\frac{p_F(\tau)}{p_B(\tau)}.
\end{align*}

For FM-GFN, we have
\begin{align*}
    \log\frac{p_F(\tau)}{p_B(\tau)}
    &=\sum_{t=1}^{n(\tau)}
    \log\frac{F^{\rightarrow}(s_t)}{F^{\leftarrow}(s_{t-1})}
    +\log\frac{Z}{R(s_{n(\tau)})}\\
    &=\sum_{t=1}^{n(\tau)} h_{\bm\theta}(s_t)
    +\log\frac{Z}{F^{\leftarrow}(s_0)},
\end{align*}
where the last equality uses $F^{\leftarrow}(s_{n(\tau)})=R(s_{n(\tau)})$.

For DB-GFN, we have
\begin{align*}
    \log\frac{p_F(\tau)}{p_B(\tau)}
    =\sum_{t=1}^{n(\tau)} h_{\bm\theta}(s_{t-1}\to s_t)
    +\log\frac{Z}{F(s_0;\bm\theta)}.
\end{align*}

For TB-GFN, we have
\begin{align*}
    \log\frac{p_F(\tau)}{p_B(\tau)}
    =h_{\bm\theta}(\tau)+\log\frac{Z}{\widehat Z(\bm\theta)}.
\end{align*}

In every case, the log-density ratio is the sum of the relevant imbalances
along $\tau$ plus a term that does not depend on $\tau$.  Since
$\sum_{\tau\in\mathcal T}(p_F(\tau)-p_B(\tau))=0$, this constant cancels in
the symmetrized expression.  Therefore,
\begin{align*}
    &\sum_{\tau\in\mathcal T}\left(p_F(\tau)-p_B(\tau)\right)
    \log\frac{p_F(\tau)}{p_B(\tau)}\\
    &\qquad=\sum_{\tau\in\mathcal T}
    \left(p_F(\tau)-p_B(\tau)\right)
    \sum_{o\in\tau}h_{\bm\theta}(o)\\
    &\qquad=\sum_{o\in\mathcal{O}}
    \left(\nu_F(o)-\nu_B(o)\right)
    h_{\bm\theta}(o).
\end{align*}
If $c_\nu=0$, the last display is zero, so the nonnegative symmetric KL
divergence is also zero.  Suppose $c_\nu>0$.  If
$\|\nu/\mu\|_\infty=+\infty$, the loss bound is immediate in the extended
sense under the coverage convention preceding
Theorem~\ref{thm:upper_bound_loss}.  Otherwise, finiteness of the ratio
implies that $\mu$ covers the support of $\nu$, and
\begin{align*}
    &\sum_{o\in\mathcal{O}}\left(\nu_F(o)-\nu_B(o)\right)
    h_{\bm\theta}(o)\\
    &\qquad\leq\sum_{o\in\mathcal{O}}
    \left|\nu_F(o)-\nu_B(o)\right|
    |h_{\bm\theta}(o)|\\
    &\qquad=c_\nu\sum_{o\in\mathcal{O}}\nu(o)
    |h_{\bm\theta}(o)|\\
    &\qquad\leq c_\nu\sqrt{\sum_{o\in\mathcal{O}}\nu(o)
    h_{\bm\theta}(o)^2}\\
    &\qquad=c_\nu\sqrt{\sum_{o\in\mathcal{O}}\mu(o)
    \frac{\nu(o)}{\mu(o)}
    h_{\bm\theta}(o)^2}\\
    &\qquad\leq c_\nu\sqrt{\left\|\frac{\nu}{\mu}\right\|_\infty
    \sum_{o\in\mathcal{O}}\mu(o)
    h_{\bm\theta}(o)^2}\\
    &\qquad=c_\nu\sqrt{2\left\|\frac{\nu}{\mu}\right\|_\infty
    \mathcal{L}_{\mathcal{O}}(\bm{\theta})}.
\end{align*}
\hfill $\square$\par

\section{Proofs for Section~\ref{sec:stationary}}
\label{app:proof_5}

All functions below are evaluated on their relevant supports, and all flows
and policy probabilities are positive there.

\paragraph{Proof of Proposition~\ref{prop:tabular-zero-residual}.}
We show for each objective that the tabular tangent space contains the
current imbalance pointwise.  The claim for every $\mu$ then follows
immediately. For DB and TB, first hold the backward policy fixed. If independent backward coordinates are included in $\bm\theta$, assign zero to their direction; the same construction remains in the enlarged tangent space.

\emph{FM.}
Give every edge $e$ an independent coordinate $z_e=\log F(e)$.  For a
non-initial state $s$, the FM imbalance is
\[
h(s)=\log F^{\rightarrow}(s)-\log F^{\leftarrow}(s).
\]
For an incoming edge $e=(r\to s)$ and an outgoing edge $e=(s\to r)$,
respectively,
\[
\frac{\partial h(s)}{\partial z_e}
=\frac{F(e)}{F^{\rightarrow}(s)}>0,
\qquad
\frac{\partial h(s)}{\partial z_e}
=-\frac{F(e)}{F^{\leftarrow}(s)},
\]
where the outgoing-edge term is absent when $s$ is terminal.  Choose one
incoming edge $e_s=(p(s)\to s)$ for every training state and order the states
topologically.  In the Jacobian submatrix formed by these columns, column
$e_s$ has a nonzero diagonal entry in row $s$ and can have one additional
nonzero entry only in the earlier row $p(s)$.  This triangular submatrix has
nonzero diagonal, so the Jacobian has full row rank.  Its range is therefore
all functions on the FM training states, including $h$.

\emph{DB.}
Use independent log-state-flow coordinates for nonterminal states and
independent softmax logits for the forward policy at each nonterminal state.
Let $h(s,s')$ be the DB imbalance.  Set $V(x)=0$ for terminal states and
define, in reverse topological order,
\begin{align}
V(s)
&=\mathbb E_{s'\sim P_F(\cdot\mid s)}[h(s,s')+V(s')],
\label{eq:db-value-recursion}\\
A(s,s')
&=h(s,s')+V(s')-V(s).
\label{eq:db-advantage}
\end{align}
Then
\[
\mathbb E_{s'\sim P_F(\cdot\mid s)}[A(s,s')]=0,
\qquad
h(s,s')=V(s)-V(s')+A(s,s').
\]
A direction $V$ in the log-state-flow coordinates contributes
$V(s)-V(s')$, with $V(s')=0$ for a terminal child.  A logit direction $A$
contributes
\[
A(s,s')-
\mathbb E_{r\sim P_F(\cdot\mid s)}[A(s,r)]=A(s,s').
\]
The combined direction is exactly $h(s,s')$ on every edge, proving that the
DB imbalance lies in the tabular tangent space.

\emph{TB.}
Use the coordinate $z=\log\widehat Z$ and independent softmax logits for the
forward policy.  For each edge, define
\[
\ell(s,s'):=\log\frac{P_F(s'\mid s)}{P_B(s\mid s')}
-\begin{cases}
\log R(s'), & s'\in\mathcal X,\\
0, & s'\notin\mathcal X.
\end{cases}
\]
The reward term is evaluated only at terminal states. Every complete trajectory $\tau=(s_0,\ldots,s_{n(\tau)})$ has the additive
imbalance
\begin{equation}
h(\tau)=z+\sum_{t=1}^{n(\tau)}\ell(s_{t-1},s_t).
\label{eq:tb-additive-imbalance}
\end{equation}
Set $V(x)=0$ for every terminal state and define, in reverse topological
order,
\begin{align}
V(s)
&=\mathbb E_{s'\sim P_F(\cdot\mid s)}[\ell(s,s')+V(s')],
\label{eq:tb-value-recursion}\\
A(s,s')
&=\ell(s,s')+V(s')-V(s).
\label{eq:tb-advantage}
\end{align}
The advantage is centered under $P_F(\cdot\mid s)$.  Since
$\ell(s,s')=V(s)-V(s')+A(s,s')$, summing along any trajectory gives
\begin{equation}
h(\tau)=z+V(s_0)+\sum_{t=1}^{n(\tau)} A(s_{t-1},s_t).
\label{eq:tb-telescoping}
\end{equation}
Choose the direction in the normalizer coordinate to be $z+V(s_0)$ and the
logit direction on edge $(s,s')$ to be $A(s,s')$.  Centering makes the
corresponding softmax-score direction equal to $A(s,s')$, so the resulting
trajectory direction is exactly $h(\tau)$.

This argument does not require equal trajectory lengths.  The value is
defined once per state, which makes merged paths consistent, and assigning
zero value to every terminal handles multiple terminal states.  Thus the TB
imbalance lies in the tabular tangent space even though the number of
trajectories may exceed the number of parameters.

For FM, DB, and TB, we have proved
$h_{\bm\theta}\in\mathsf T_{\bm\theta}$ pointwise for every
$\bm\theta$.  Hence its orthogonal residual is zero in every $L_2(\mu)$ and
the supremum defining $\epsilon_\mu$ is zero.
\hfill $\square$\par

\paragraph{Proof of Lemma~\ref{lemma:local-coefficient}.}
Fix $\bm\theta$ and define the linear operator
\[
\Psi_{\bm\theta}:\mathbb R^d\to L_2(\mu),
\qquad
(\Psi_{\bm\theta}\bm u)(o)
=\bm u^\top\psi_{\bm\theta}(o).
\]
By construction,
$\operatorname{Range}(\Psi_{\bm\theta})=\mathsf T_{\bm\theta}$.
Its adjoint and Gram matrix are
\[
\Psi_{\bm\theta}^*f
=\mathbb E_\mu[\psi_{\bm\theta}(o)f(o)],
\qquad
\Psi_{\bm\theta}^*\Psi_{\bm\theta}
=\mathcal I_\mu(\bm\theta).
\]
The standard pseudoinverse formula for the orthogonal projection onto the
range of a finite-dimensional linear operator gives
\begin{equation}
\Pi_{\mathsf T_{\bm\theta}}^\mu
=\Psi_{\bm\theta}\mathcal I_\mu(\bm\theta)^\dagger
\Psi_{\bm\theta}^*.
\label{eq:tangent-projection-pseudoinverse}
\end{equation}
By the square-integrability and differentiation conventions in
Section~\ref{sec:preliminaries}, the operator above is bounded and
differentiation may pass through the loss expectation, with $\mu$ and any
separate $\bm\varphi$ fixed. Hence
\[
\Psi_{\bm\theta}^*h_{\bm\theta}
=\mathbb E_\mu[\psi_{\bm\theta}(o)h_{\bm\theta}(o)]
=\nabla_{\bm\theta}\frac12
\mathbb E_\mu[h_{\bm\theta}(o)^2]
=\nabla_{\bm\theta}\mathcal L_{\mathcal O}(\bm\theta).
\]
Substituting this identity into
Equation~\ref{eq:tangent-projection-pseudoinverse} yields
\[
\Pi_{\mathsf T_{\bm\theta}}^\mu h_{\bm\theta}
=\Psi_{\bm\theta}\bm u_\mu^*(\bm\theta)
={\bm u_\mu^*(\bm\theta)}^\top\psi_{\bm\theta}.
\]
The Moore--Penrose solution is the minimum-norm coefficient among all
coefficients representing this projection.  Assumption~\ref{assump:fisher}
then gives
\[
\|\bm u_\mu^*(\bm\theta)\|_2
\leq\|\mathcal I_\mu(\bm\theta)^\dagger\|_{2\to2}
\|\nabla_{\bm\theta}\mathcal L_{\mathcal O}(\bm\theta)\|_2
\leq\frac{1}{\lambda}
\|\nabla_{\bm\theta}\mathcal L_{\mathcal O}(\bm\theta)\|_2.
\]
\hfill $\square$\par

\paragraph{Proof of Theorem~\ref{thm:stationarity_to_optimality}.}
If $c_\nu=0$, then $\nu_F(o)=\nu_B(o)$ for every $o$, and the signed
occupancy bound established in the proof of
Theorem~\ref{thm:upper_bound_loss} in Appendix~\ref{app:proof_4}
implies that the symmetric
KL divergence is zero.  Suppose henceforth that $c_\nu>0$, so the discrepancy
distribution $\nu$ is well defined.

If $\|\nu/\mu\|_\infty=+\infty$, the claimed bound holds in the extended
sense under the coverage convention preceding
Theorem~\ref{thm:upper_bound_loss}.  Otherwise, finiteness of this ratio
implies that $\mu$ covers the support of $\nu$, so the changes of measure
below are valid.

The same signed occupancy bound from Appendix~\ref{app:proof_4} and
the definition of $\nu$ give
\begin{align*}
&\mathrm{KL}(p_F^\top\|p_R)+\mathrm{KL}(p_R\|p_F^\top)\\
&\qquad\leq\sum_{o\in\mathcal O}
\bigl(\nu_F(o)-\nu_B(o)\bigr)h_{\bm\theta}(o)
\leq c_\nu\mathbb E_{o\sim\nu}|h_{\bm\theta}(o)|.
\end{align*}
By Definition~\ref{def:local-residual} and
Lemma~\ref{lemma:local-coefficient},
\[
h_{\bm\theta}(o)
=r_{\mu,\bm\theta}(o)
+{\bm u_\mu^*(\bm\theta)}^{\!\top}\psi_{\bm\theta}(o).
\]
For the residual component, Cauchy--Schwarz and importance sampling yield
\begin{align*}
\mathbb E_\nu|r_{\mu,\bm\theta}|
&\leq\sqrt{\mathbb E_\nu[r_{\mu,\bm\theta}^2]}\\
&=\sqrt{\mathbb E_\mu\!\left[
\frac{\nu}{\mu}r_{\mu,\bm\theta}^2\right]}\\
&\leq\sqrt{\left\|\frac{\nu}{\mu}\right\|_\infty}
\|r_{\mu,\bm\theta}\|_\mu\\
&\leq\sqrt{\left\|\frac{\nu}{\mu}\right\|_\infty}
\epsilon_\mu.
\end{align*}
For the tangent component, Assumption~\ref{assump:score} and
Lemma~\ref{lemma:local-coefficient} imply
\begin{align*}
\mathbb E_\nu\left|
{\bm u_\mu^*(\bm\theta)}^{\!\top}\psi_{\bm\theta}(o)\right|
&\leq\|\bm u_\mu^*(\bm\theta)\|_2
\mathbb E_\nu\|\psi_{\bm\theta}(o)\|_2\\
&\leq\frac{M}{\lambda}
\|\nabla_{\bm\theta}\mathcal L_{\mathcal O}(\bm\theta)\|_2.
\end{align*}
Combining the two component bounds proves
\[
\mathrm{KL}(p_F^\top\|p_R)+\mathrm{KL}(p_R\|p_F^\top)
\leq c_\nu\left(
\sqrt{\left\|\frac{\nu}{\mu}\right\|_\infty}
\epsilon_\mu
+\frac{M}{\lambda}
\|\nabla_{\bm\theta}\mathcal L_{\mathcal O}(\bm\theta)\|_2
\right).
\]
\hfill $\square$\par

\section{Proofs for Section~\ref{sec:convergence-rates}}
\label{app:convergence-proofs}

We use the parameter blocks and conditional expectations defined in
Section~\ref{sec:convergence-rates}. Constants denoted by $C$ may increase
between inequalities but are independent of $t$ and $\bar\epsilon$.

\subsection{Geometry and a common descent estimate}
\label{app:convergence-geometry}

\paragraph{Proof of Lemma~\ref{lemma:approximate-PL}.}
The setup in Section~\ref{sec:convergence-rates} makes
$\bm u\mapsto\bm u^\top\nabla_{\bm\theta}h$ a bounded map from
$\mathbb R^d$ into $L_2(\mu)$. Its range is finite-dimensional and hence
closed, even when $\mathcal O$ is infinite. The interchange of
differentiation and summation assumed in Section~\ref{sec:preliminaries} gives
$\nabla_{\bm\theta}\mathcal L_{\mathcal O}
=\mathbb E_\mu[h\nabla_{\bm\theta}h]$, so the projection identity of
Lemma~\ref{lemma:local-coefficient} still applies.
At the current parameters, write
$\bm G=\nabla_{\bm\theta}\mathcal L_{\mathcal O}$,
$I=\mathcal I_\mu(\bm\theta)$, and let $\Pi$ be the $L_2(\mu)$ tangent
projection. Lemma~\ref{lemma:local-coefficient} gives the coefficient
$\bm u^*=I^\dagger\bm G$, so
\[
    \|\Pi h\|_\mu^2
    ={\bm u^*}^{\!\top}I\bm u^*
    =\bm G^\top I^\dagger\bm G.
\]
Since $r=h-\Pi h$ is orthogonal to $\Pi h$, Pythagoras and
Assumption~\ref{assump:uniform-tangent-geometry} yield
\[
    2\mathcal L_{\mathcal O}
    =\|r\|_\mu^2+\bm G^\top I^\dagger\bm G
    \leq\bar\epsilon^2+\frac{\|\bm G\|_2^2}{\lambda}.
\]
\hfill $\square$\par

\paragraph{A common one-step estimate.}
\label{app:convergence-controlled}
We prove a recursion under the conditions of
Theorem~\ref{thm:sgd-general-PB}. Without a separate backward update, set
$z_t=0$ and omit the backward sensitivity condition; the sampling bounds
then reduce to Assumption~\ref{assump:controlled-sampling}. With fixed
sampling, omit the sampling block and set its constants to zero.
Write $x_t=\|\nabla_{\bm\theta}\mathcal L_t\|_2$,
$\Delta_t^\theta=\bm\theta_{t+1}-\bm\theta_t$, and
$\Delta_t^\varphi=\bm\varphi_{t+1}-\bm\varphi_t$.
Joint smoothness gives
\begin{align}
    \mathbb E_t\mathcal L_{t+1}
    \leq{}&\mathcal L_t-\eta_t x_t^2
    +\langle\nabla_{\bm\omega}\mathcal L_t,
       \mathbb E_t\Delta_t^\omega\rangle
    +\langle\nabla_{\bm\varphi}\mathcal L_t,
       \mathbb E_t\Delta_t^\varphi\rangle\notag\\
    &+\frac L2\mathbb E_t\big[
       \|\Delta_t^\theta\|_2^2+\|\Delta_t^\omega\|_2^2
       +\|\Delta_t^\varphi\|_2^2\big].
    \label{eq:common-one-step-proof}
\end{align}
All gradients on the right are $\mathcal F_t$-measurable; no independence
between updates is needed. The sensitivity and update bounds, together
with $\|\mathbb E_t\Delta_t^\varphi\|_2\leq z_t$, bound the two drift
terms by
\[
    C_\omega K_\theta\eta_t\sqrt{2\mathcal L_t}\,x_t
    +(C_\omega K_\varphi+C_\varphi)\sqrt{2\mathcal L_t}\,z_t.
\]
Lemma~\ref{lemma:approximate-PL} implies
$\sqrt{2\mathcal L_t}\leq\bar\epsilon+x_t/\sqrt\lambda$.
With $\delta=1-\rho>0$, the one-step bound becomes
\[
    \mathbb E_t\mathcal L_{t+1}
    \leq\mathcal L_t-\delta\eta_t x_t^2
       +C\eta_t\bar\epsilon x_t+C(\bar\epsilon+x_t)z_t
       +C\eta_t^2(x_t^2+\sigma^2)+Cz_t^2.
\]
Young's inequality bounds the mixed terms as
\[
    C\eta_t\bar\epsilon x_t+C(\bar\epsilon+x_t)z_t
    \leq\frac\delta4\eta_t x_t^2
       +C\eta_t\bar\epsilon^2+C\frac{z_t^2}{\eta_t}.
\]
Choose $\eta_*>0$ sufficiently small that $\eta_*\leq1$ and the quadratic
term $C\eta_t^2x_t^2$ is at most $\delta\eta_tx_t^2/4$ whenever
$\eta_t\leq\eta_*$. Absorbing $Cz_t^2$ and the fixed noise constant gives
\begin{equation}
    \mathbb E_t\mathcal L_{t+1}
    \leq\mathcal L_t-\frac\delta2\eta_t x_t^2
       +C\eta_t\bar\epsilon^2+C\eta_t^2
       +C\frac{z_t^2}{\eta_t}.
    \label{eq:common-descent-proof}
\end{equation}
Using $x_t^2\geq\lambda(2\mathcal L_t-\bar\epsilon^2)$ yields
\begin{equation}
    \mathbb E_t\mathcal L_{t+1}
    \leq(1-c\eta_t)\mathcal L_t
       +C\eta_t\bar\epsilon^2+C\eta_t^2
       +C\frac{z_t^2}{\eta_t},
    \qquad c=\lambda(1-\rho)>0.
    \label{eq:general-PB-recursion}
\end{equation}
Decrease $\eta_*$ if needed so that $c\eta_*\leq1/2$.
Finite second moments of the updates imply integrability of $x_t^2$
by conditional Jensen and~(\ref{eq:balance-sgd}), and of $z_t^2$ by
definition. Starting from $\mathbb E\mathcal L_0<\infty$,
(\ref{eq:general-PB-recursion}) therefore gives finite expected loss at
each finite time and justifies taking total expectation.

\subsection{Fixed and controlled sampling rates}
\label{app:convergence-fixed}

\paragraph{A scalar comparison.}
Choose $a>1/c$, $t_0\geq\max\{1,a/\eta_*\}$, and
$\eta_t=a/(t+t_0)$. For any $0<\alpha\leq1$, a nonnegative sequence
satisfying
\[
    u_{t+1}\leq(1-c\eta_t)u_t+\frac C{(t+t_0)^{1+\alpha}}
\]
obeys $u_t\leq D/(t+t_0)^\alpha$ for sufficiently large $D$.
Indeed, choose $D\geq t_0^\alpha u_0$ and $(ca-\alpha)D\geq C$.
With $k=t+t_0$, induction and convexity of $k\mapsto k^{-\alpha}$ give
\[
    u_{t+1}\leq\frac D{k^\alpha}-\frac{caD-C}{k^{1+\alpha}}
    \leq\frac D{k^\alpha}-\frac{\alpha D}{k^{1+\alpha}}
    \leq\frac D{(k+1)^\alpha}.
\]
The contraction factors are nonnegative by the choice of $t_0$.

\paragraph{Proofs of Theorems~\ref{thm:sgd-fixed-sampling}
and~\ref{thm:sgd-controlled-sampling}.}
In~(\ref{eq:general-PB-recursion}), set $z_t=0$; fixed sampling also gives
$\rho=0$. Take total expectation and choose $A\geq C/c$ for the
coefficient of $\eta_t\bar\epsilon^2$. Then
$u_t=(\mathbb E\mathcal L_t-A\bar\epsilon^2)_+$ satisfies
\[
    u_{t+1}\leq(1-c\eta_t)u_t+C\eta_t^2.
\]
The scalar comparison with $\alpha=1$ proves~(\ref{eq:sgd-loss-rate}).
FM is included because both policies are induced by $\bm\theta$; joint
backward training is included because its parameters belong to the same
descent block. These arguments require finite-dimensional parameters,
not a finite object space. The tabular zero-residual specialization in
Section~\ref{sec:convergence-rates} is restricted to finite DAGs.
\hfill $\square$\par

\paragraph{Conversion to terminal accuracy.}
We first justify the bound of Theorem~\ref{thm:upper_bound_loss} under
Assumption~\ref{assump:terminal-control} on a countable DAG. Fix the
current parameters and suppress $t$. Almost-sure termination makes
$p_F,p_B$ probability distributions on finite complete trajectories,
with backward terminal marginal $p_R$. For FM and DB the integrability
condition is
\[
    \mathbb E_{\tau\sim p_F}\sum_{o\in\tau}|h(o)|
    +\mathbb E_{\tau\sim p_B}\sum_{o\in\tau}|h(o)|<\infty.
\]
Tonelli's theorem identifies this quantity with
$\sum_o(\nu_F(o)+\nu_B(o))|h(o)|$. It therefore justifies the signed
sum interchange in Appendix~\ref{app:proof_4}. The telescoping identities
hold on each finite trajectory, and their trajectory-independent
constants cancel because $p_F,p_B$ are normalized. Thus the same signed
bound holds for terminal symmetric KL.

For TB, there is no interchange between trajectory and object sums.
The identity $\log(p_F/p_B)=h+\log(Z/\widehat Z)$ and normalization give
the signed bound directly. Its absolute convergence follows from
\[
    \sum_\tau|p_F(\tau)-p_B(\tau)|\,|h(\tau)|
    \leq c_\nu\sqrt{2\bar C_\nu\mathcal L_{\mathcal O}}<\infty
\]
when $c_\nu>0$; when $c_\nu=0$, $p_F=p_B$.
In all three cases, the change of measure and Cauchy--Schwarz argument
of Theorem~\ref{thm:upper_bound_loss} is valid on the countable object
space. With the deterministic bounds from
Assumption~\ref{assump:terminal-control}, Jensen's inequality now gives
\begin{align*}
    \mathbb E\big[\mathrm{KL}(p_{F,T}^\top\|p_R)
       +\mathrm{KL}(p_R\|p_{F,T}^\top)\big]
    &\leq\bar B\sqrt{2\bar C_\nu}\,
       \mathbb E\sqrt{\mathcal L_T}\\
    &\leq\bar B\sqrt{2\bar C_\nu\,\mathbb E\mathcal L_T}.
\end{align*}
This also holds when $c_{\nu,T}=0$, since the terminal discrepancy is
then zero. Substituting~(\ref{eq:sgd-loss-rate}) proves
(\ref{eq:sgd-terminal-rate}); the same conversion applies to the bounds below.

For TB, $c_{\nu,t}\leq2$ follows from normalization alone. For FM and
DB, Tonelli's theorem gives
\[
    c_{\nu,t}\leq\sum_o(\nu_{F,t}(o)+\nu_{B,t}(o))
    =\mathbb E_{\tau\sim p_{F,t}}n(\tau)
       +\mathbb E_{\tau\sim p_{B,t}}n(\tau).
\]
A deterministic uniform bound on this sum of expectations, conditional
on the current parameters, suffices for $\bar B$. Finiteness at each
iterate alone does not give a uniform constant. On a finite DAG of
maximum path length $H$, one may take $\bar B=2H$ for FM and DB;
the absolute-imbalance integrability condition also holds automatically.

\subsection{Separate backward updates}
\label{app:convergence-general}

\paragraph{Proof of Theorem~\ref{thm:sgd-general-PB}.}
Taking total expectation in~(\ref{eq:general-PB-recursion}) uses
$d_t^2=\mathbb E z_t^2=\mathbb E\|\bm\varphi_{t+1}-\bm\varphi_t\|_2^2$.
With the stepsize and residual shift from
Appendix~\ref{app:convergence-fixed},
\begin{equation}
    u_{t+1}\leq(1-c\eta_t)u_t+C\eta_t^2
       +C\frac{d_t^2}{\eta_t},
    \qquad u_t=(\mathbb E\mathcal L_t-A\bar\epsilon^2)_+.
    \label{eq:PB-expectation-recursion}
\end{equation}
If $d_t=o(\eta_t)$, the additive term is $\eta_t\xi_t$ for a deterministic
sequence $\xi_t\to0$. For any $\varepsilon>0$, eventually
$\xi_t\leq c\varepsilon$, so
\[
    (u_{t+1}-\varepsilon)_+
    \leq(1-c\eta_t)(u_t-\varepsilon)_+.
\]
Since $\sum_t\eta_t=\infty$, the product of contraction factors tends to
zero. Hence $\limsup_tu_t\leq\varepsilon$, and letting
$\varepsilon\downarrow0$ proves the first claim.
If $d_t=O(\eta_t^{3/2})$, the additional forcing is $O(\eta_t^2)$;
the scalar comparison with $\alpha=1$ proves the second claim, absorbing
any finite initial segment into the constant. The terminal bounds follow
from Appendix~\ref{app:convergence-fixed}.
\hfill $\square$\par

\paragraph{Intermediate decay rates.}
Under the hypotheses of Theorem~\ref{thm:sgd-general-PB}, if
$d_t=O(t^{-b})$ with $1<b<3/2$, then
\[
    \mathbb E\mathcal L_T=O(\bar\epsilon^2+T^{-2(b-1)}).
\]
Indeed, the forcing in~(\ref{eq:PB-expectation-recursion}) is
$O((t+t_0)^{-2}+(t+t_0)^{1-2b})$.
The scalar comparison applies with $\alpha=2(b-1)\in(0,1)$.
Under Assumption~\ref{assump:terminal-control}, the expected terminal symmetric KL is consequently
$O(\bar\epsilon+T^{-(b-1)})$.

\subsection{Sufficient conditions for sensitivity, sampling, and coverage}
\label{app:convergence-regularity}

The following examples verify individual conditions used above; they do
not replace the remaining hypotheses of the theorems. The derivative
identities use the interchange convention in Section~\ref{sec:preliminaries},
with the external arguments made explicit in Section~\ref{sec:convergence-rates}.
For a countable object space, a sufficient condition for that interchange
is local domination of each derivative of the summand
$\mu_{\bm\omega}(o)h_{\bm\theta,\bm\varphi}(o)^2/2$
by a summable sequence on a parameter neighborhood, with finite loss at
one point of that neighborhood. The mean value theorem and dominated
convergence then justify differentiating the sum in each relevant
coordinate. The square-integrability of $h$ and its $\bm\theta$-derivatives
also ensures that the tangent operator in Appendix~\ref{app:proof_5} is
bounded and that $\mathbb E_\mu[h\nabla_{\bm\theta}h]$ is finite by
Cauchy--Schwarz. On a finite object space these conditions hold for
continuously differentiable, finite imbalances and sampling probabilities
on a parameter neighborhood.

\paragraph{Direct sensitivity bounds.}
Holding sampling fixed, Cauchy--Schwarz gives
\[
    \|\nabla_{\bm\varphi}\mathcal L_{\mathcal O}\|_2
    =\|\mathbb E_\mu[h\nabla_{\bm\varphi}h]\|_2
    \leq\sqrt{2\mathcal L_{\mathcal O}}
       \left(\mathbb E_\mu\|\nabla_{\bm\varphi}h\|_2^2\right)^{1/2}.
\]
A uniform bound on the second factor supplies $C_\varphi$. For DB,
$\nabla_{\bm\varphi}h$ is the negative backward log-policy score on an
edge; for TB it is the sum of these negative scores along the trajectory.
Uniformly bounded edge scores suffice for DB. For TB they suffice if
$\mathbb E_{\tau\sim\mu_{\bm\omega_t}}n(\tau)^2$ is uniformly bounded;
this follows in particular from bounded trajectory lengths. This second
moment under the training distribution is distinct from the occupancy
mass condition under $p_{F,t},p_{B,t}$. For
neural policies, these bounds concern derivatives with respect to the
network parameters and do not follow from softmax normalization alone.

For differentiable sampling probabilities on a fixed positive support,
let $s_\omega(o)=\nabla_{\bm\omega}\log\mu_{\bm\omega}(o)$. If
$\|s_\omega(o)\|_2\leq S_\omega$ and $|h(o)|\leq B_h$ uniformly, then
\[
    \|\nabla_{\bm\omega}\mathcal L_{\mathcal O}\|_2
    =\frac12\|\mathbb E_\mu[h^2s_\omega]\|_2
    \leq\frac{S_\omega B_h}{2}\mathbb E_\mu|h|
    \leq\frac{S_\omega B_h}{2}\sqrt{2\mathcal L_{\mathcal O}}.
\]
Thus $C_\omega=S_\omega B_h/2$ suffices. Bounded sampling scores alone
do not ensure this square-root bound for unbounded imbalances.

\paragraph{Affine sampling updates.}
Suppose
$\Delta_t^\omega=A_t\Delta_t^\theta+B_t\Delta_t^\varphi$,
where $A_t,B_t$ are $\mathcal F_t$-measurable and their operator norms
are bounded by $k_\theta,k_\varphi$. Conditional Jensen and the balance
oracle give
\begin{align*}
    \|\mathbb E_t\Delta_t^\omega\|_2
    &\leq k_\theta\eta_t\|\nabla_{\bm\theta}\mathcal L_t\|_2
       +k_\varphi z_t,\\
    \mathbb E_t\|\Delta_t^\omega\|_2^2
    &\leq2k_\theta^2\eta_t^2
       (\|\nabla_{\bm\theta}\mathcal L_t\|_2^2+\sigma^2)
       +2k_\varphi^2z_t^2.
\end{align*}
Hence Assumption \ref{assump:controlled-sampling-extension} holds with
$K_\theta=k_\theta$, $K_\varphi=k_\varphi$, and
$K_2=2\max\{k_\theta^2,k_\varphi^2\}$; omit the backward term when
that block is absent. This includes parameter copying
$\bm\omega_t=\bm\theta_t$, but the sensitivity bound and $\rho<1$
remain separate requirements. A nonlinear parameter map must satisfy
the conditional-mean bound directly: a Lipschitz increment bound alone
controls second moments, not the mean's Taylor remainder.

\subsection{Sampling stabilisation: exploration, replay, and stale targets}
\label{app:sampling-stabilization}

This subsection gives representative ways to verify the hypotheses of
Theorem~\ref{thm:sgd-controlled-sampling}, relating familiar exploration,
replay, and target-update mechanisms to its constants. Omit any separate
backward update and write $x_t=\|\nabla_{\bm\theta}\mathcal L_t\|_2$. Collect
lag, refresh, and perturbation terms in
\begin{align}
    \|\mathbb E_t\Delta_t^\omega\|_2
       &\leq K_\theta\eta_t x_t+b_t, \\
    \mathbb E_t\|\Delta_t^\omega\|_2^2
       &\leq K_2\eta_t^2(x_t^2+\sigma^2)+v_t.
    \label{eq:sampler-drift-bookkeeping}
\end{align}
When $b_t=v_t=0$, these are exactly
Assumption~\ref{assump:controlled-sampling}; otherwise they must be
absorbed into its bounds; retaining them as forcing terms requires a
corresponding extension of the descent analysis.

\paragraph{Sampling sensitivity.}
Uniform-random exploration is familiar from $\varepsilon$-greedy
policies~\citep{sutton2018reinforcement}; GFlowNet exploration also uses
noisy policies and tempering~\citep{malek2026loss}.
We analyse an $\varepsilon$-noisy mixture over complete training objects:
choose a fixed exploration distribution $q$ with probability
$\varepsilon_{\mathrm{exp}}$, and the current sampler otherwise. For
$0\leq\varepsilon_{\mathrm{exp}}<1$, let
\begin{equation}
    \mu_{\bm\omega}^{(\varepsilon)}
       =(1-\varepsilon_{\mathrm{exp}})p_{\bm\omega}
          +\varepsilon_{\mathrm{exp}}q,\qquad
    \mathcal L_{\varepsilon}
       =(1-\varepsilon_{\mathrm{exp}})\mathcal L_{\mathrm{fresh}}
          +\varepsilon_{\mathrm{exp}}\mathcal L_q.
    \label{eq:epsilon-noisy-mixture}
\end{equation}
If
$\|\nabla_{\bm\omega}\mathcal L_{\mathrm{fresh}}\|_2
 \leq C_{\omega,\mathrm{fresh}}\sqrt{2\mathcal L_{\mathrm{fresh}}}$,
then
\begin{equation}
    \|\nabla_{\bm\omega}\mathcal L_{\varepsilon}\|_2
       \leq \sqrt{1-\varepsilon_{\mathrm{exp}}}\,
          C_{\omega,\mathrm{fresh}}\sqrt{2\mathcal L_{\varepsilon}}.
    \label{eq:epsilon-noisy-sensitivity}
\end{equation}
Thus fixed exploration changes $C_\omega$, while the underlying update
rule determines $K_\theta$ and $K_2$. For exploration mixed at each
action, the induced trajectory law is a product of local mixtures;
its sensitivity follows from the resulting trajectory score.

Temperature scaling, also called softmax or Boltzmann
exploration~\citep{sutton2018reinforcement,malek2026loss}, flattens a
sampling distribution at temperatures above one.
For logits $\ell_{\bm\omega}(o)$ and fixed $\tau>0$ such that
$Z_\tau(\bm\omega):=\sum_o\exp(\ell_{\bm\omega}(o)/\tau)<\infty$
on the parameter region under consideration, and with differentiation
permitted through this normalizing sum,
\begin{equation}
    p_{\bm\omega}^{(\tau)}(o)
      =\frac{\exp(\ell_{\bm\omega}(o)/\tau)}
             {\sum_{o'}\exp(\ell_{\bm\omega}(o')/\tau)},\qquad
    \nabla_{\bm\omega}\log p_{\bm\omega}^{(\tau)}(o)
       =\frac1\tau\left(
          \nabla_{\bm\omega}\ell_{\bm\omega}(o)
          -\mathbb E_{p_{\bm\omega}^{(\tau)}}
             \nabla_{\bm\omega}\ell_{\bm\omega}(o')\right).
    \label{eq:tempered-score}
\end{equation}
Both requirements are automatic for finitely many continuously
differentiable logits. On a countable object space, a uniform local bound
$\|\nabla_{\bm\omega}\ell_{\bm\omega}(o)\|_2\leq G$ together with
$Z_\tau(\bm\omega_0)<\infty$ at a reference point suffices locally:
on a convex neighborhood of radius $r$ about $\bm\omega_0$,
\[
    e^{\ell_{\bm\omega}(o)/\tau}
    \leq e^{Gr/\tau}e^{\ell_{\bm\omega_0}(o)/\tau},\qquad
    \left\|\nabla_{\bm\omega}e^{\ell_{\bm\omega}(o)/\tau}\right\|_2
    \leq\frac G\tau e^{Gr/\tau}e^{\ell_{\bm\omega_0}(o)/\tau}.
\]
These summable bounds justify both normalization and its derivative.
Finiteness at one temperature does not imply finiteness at a higher
temperature: logits $\ell(n)=-2\log n$, $n\geq1$, are normalizable at
$\tau=1$ but not at $\tau=2$.

If $\|\nabla_{\bm\omega}\ell_{\bm\omega}(o)\|_2\leq G$ and $|h(o)|\leq B_h$
uniformly, the score bound $2G/\tau$ and Appendix~\ref{app:convergence-regularity}
give $C_{\omega,\tau}=B_hG/\tau$. For actionwise tempering on trajectories
of length at most $H$, summing local scores gives the bound $B_hHG/\tau$.
Fixed tempering leaves $K_\theta,K_2$ determined by the update rule.

Experience replay reuses previously collected data, as in
DQN~\citep{mnih2015human}; prioritized experience replay further changes
their sampling weights~\citep{schaul2016prioritized}. Freezing both buffer
contents and weights gives fixed sampling for the associated empirical
objective. For a fixed mixture of replay and fresh samples,
$\mu_{\bm\omega}=(1-\beta)q+\beta p_{\bm\omega}$, the same argument gives
\begin{equation}
    C_{\omega,\mathrm{mix}}
       =\sqrt{\beta}\,C_{\omega,\mathrm{fresh}},
    \qquad 0\leq\beta\leq1.
    \label{eq:replay-mixture-sensitivity}
\end{equation}
Coverage and the effective Fisher lower bound must be recomputed under
the resulting distribution. Changing buffer contents, exploration
weights, temperatures, or mixture weights adds drift. A differentiable
parameterisation of these changes allows their increments to enter
$b_t,v_t$; direct changes of the sampling measure can instead be bounded
in the loss. For replay, the oracle is evaluated against the empirical
or mixture objective actually used.

\paragraph{Update size and stale targets.}
If the sampler follows damped balance increments,
\[
    \bm\omega_{t+1}=\bm\omega_t+
        \alpha(\bm\theta_{t+1}-\bm\theta_t),\qquad 0<\alpha\leq1,
\]
then
\[
    \mathbb E_t\Delta_t^\omega
       =-\alpha\eta_t\nabla_{\bm\theta}\mathcal L_t,\qquad
    \mathbb E_t\|\Delta_t^\omega\|_2^2
       \leq\alpha^2\eta_t^2(x_t^2+\sigma^2),
\]
so $K_\theta=\alpha$ and $K_2=\alpha^2$. This provides a direct way to
limit sampler motion by scaling each new balance increment.

Soft target updates, often called Polyak or exponential moving average
(EMA) updates, are used in DDPG~\citep{lillicrap2016continuous}.
Here we apply that update rule to a sampling copy; the cited RL target
networks are used to form learning targets. With
$\bm\omega_{t+1}=(1-\alpha)\bm\omega_t+\alpha\bm\theta_{t+1}$
and $\bm e_t=\bm\theta_t-\bm\omega_t$, the increment satisfies
\begin{equation}
    \Delta_t^\omega=\alpha\bm e_t+\alpha\Delta_t^\theta,\qquad
    \|\mathbb E_t\Delta_t^\omega\|_2
       \leq\alpha\|\bm e_t\|_2+\alpha\eta_t x_t.
    \label{eq:polyak-sampling-lag}
\end{equation}
The tracking condition
$\|\bm e_t\|_2\leq\kappa\eta_t x_t$ reduces this to
$K_\theta=\alpha(1+\kappa)$ and
$K_2=2\alpha^2(1+\kappa^2)$. Polyak averaging alone does not provide
this tracking condition.

Periodic hard target updates in DQN~\citep{mnih2015human} motivate a
lazy sampling copy, held fixed between synchronisations. Let
$I_t$ indicate a copy step and
$\bm b_t=\bm\theta_t-\bm\omega_t$. Then
\begin{equation}
  \begin{aligned}
    \Delta_t^\omega&=I_t(\bm b_t+\Delta_t^\theta),\qquad
    \|\mathbb E_t\Delta_t^\omega\|_2
       \leq I_t(\|\bm b_t\|_2+\eta_t x_t),\\
    \mathbb E_t\|\Delta_t^\omega\|_2^2
       &\leq2I_t\|\bm b_t\|_2^2+
          2I_t\eta_t^2(x_t^2+\sigma^2).
  \end{aligned}
    \label{eq:lazy-periodic-bounds}
\end{equation}
For a fixed stale copy $\bm\omega_t=\bm\theta_{t-m}$,
$\Delta_t^\omega=\Delta_{t-m}^\theta$. A bound on $\bm b_t$ or on the
past update in terms of the current $x_t$ is therefore required before
these rules can be represented by finite $K_\theta,K_2$. In particular,
the period $m$ does not imply $K_\theta=1/m$. Conditionally zero-mean
additive update noise preserves the first-order bound and can be
accounted for by increasing $K_2,v_t$; biased noise also contributes
to $b_t$.

\section{Extension to Sub-Trajectory Balance}
\label{app:subtb}

Sub-Trajectory Balance (SubTB; \citealt{madan2023learning}) interpolates
between single-edge and complete-trajectory credit assignment. This
appendix extends the distributional, stationarity, and convergence
results of Sections~\ref{sec:objectives}--\ref{sec:convergence-rates}
using the additive structure of its imbalance. Coverage, tangent
geometry, and drift remain explicit assumptions. Throughout this
appendix, the training-object space in the unified notation of
Section~\ref{sec:preliminaries} is
$\mathcal O=\mathcal O_{\mathrm{sub}}$.

\paragraph{Objective and sampling measure.}
Let $\mathcal O_{\mathrm{sub}}$ be the set of nonempty finite directed
subtrajectories, including single edges and complete trajectories. Use a
positive state flow $F(s;\bm\theta)$ on nonterminal states and set
$F(x;\bm\theta)=R(x)$ on terminal states. For
$\sigma=(s_i,\ldots,s_j)$ with length $\ell(\sigma)=j-i\geq1$, define
\begin{equation}
\begin{aligned}
    F^{\rightarrow}(\sigma;\bm\theta)
       &:=F(s_i;\bm\theta)
          \prod_{k=i}^{j-1}P_F(s_{k+1}\mid s_k;\bm\theta),\\
    F^{\leftarrow}(\sigma;\bm\theta,\bm\varphi)
       &:=F(s_j;\bm\theta)
          \prod_{k=i}^{j-1}P_B(s_k\mid s_{k+1};\bm\varphi),\\
    h_{\bm\theta,\bm\varphi}(\sigma)
       &:=\log\frac{F^{\rightarrow}(\sigma;\bm\theta)}
                         {F^{\leftarrow}(\sigma;\bm\theta,\bm\varphi)}
       &=\sum_{k=i}^{j-1}
              h^{\mathrm{DB}}_{\bm\theta,\bm\varphi}(s_k,s_{k+1}).
\end{aligned}
\label{eq:subtb-additive}
\end{equation}
The last equality follows by telescoping the intermediate state flows. A
single edge gives DB, and a complete trajectory gives TB with
$\widehat Z=F(s_0;\bm\theta)$. As in the main text, independently
specified backward parameters are held fixed under a $\bm\theta$
derivative, whereas jointly balance-trained backward parameters are part
of $\bm\theta$. We suppress fixed parameter arguments below and write
$h_{\bm\theta}$ and $h$ when no ambiguity can arise.

For a complete trajectory $\tau=(s_0,\ldots,s_n)$, normalized length
weights can be written as
\[
    w_{ij}^{\gamma}(\tau)
       :=\frac{\gamma^{j-i}}
                    {\sum_{0\leq a<b\leq n}\gamma^{b-a}},
    \qquad 0\leq i<j\leq n,\quad\gamma>0.
\]
Here $\gamma$ denotes the SubTB weighting parameter, distinct from the
effective Fisher lower bound $\lambda$. Given a trajectory sampling law
 $q_{\bm\omega}$, the induced normalized
distribution on subtrajectories is
\begin{equation}
    \mu_{\bm\omega}(\sigma)
       :=\mathbb E_{\tau\sim q_{\bm\omega}}
           \sum_{0\leq i<j\leq n(\tau)}
             w_{ij}^{\gamma}(\tau)
             \mathbf 1\{\tau_{i:j}=\sigma\}.
    \label{eq:subtb-sampling}
\end{equation}
Its loss is the unified balance objective
$\mathcal L_{\mathcal O}(\bm\theta,\bm\omega;\bm\varphi)
=\frac12\mathbb E_{\mu_{\bm\omega}}[h_{\bm\theta,\bm\varphi}^2]$,
equivalently the expected weighted sum of subtrajectory losses. For each
fixed finite trajectory,
$\gamma\downarrow0$ gives the uniform average of its edge losses and
$\gamma\to\infty$ gives its complete-trajectory loss. The induced
sampling weights, as well as the imbalance formula, matter for these
DB/TB limits. Passing a limit through an expectation on an unbounded
trajectory space requires an integrable dominating function.
The results below also allow other normalized subtrajectory distributions,
including those constructed from incomplete episodes. For a fixed sampling
law, write $\mu=\mu_{\bm\omega}$ and suppress the fixed external
arguments of $\mathcal L_{\mathcal O}$. Sampling is held fixed when
differentiating the balance loss, as in Section~\ref{sec:convergence-rates}.

\subsection{From SubTB loss to terminal-distribution accuracy}
\label{app:subtb-distribution}

Fix a deterministic, parameter-independent rule $\mathcal P$ that
partitions every complete trajectory into consecutive nonempty blocks:
\[
    0=i_0<i_1<\cdots<i_m=n(\tau),
    \qquad
    \mathcal P(\tau)=\{\tau_{i_{r-1}:i_r}:1\leq r\leq m\}.
\]
This rule is used only to organize the proof; the training distribution
may still average over all overlapping subtrajectories. For the forward and
backward trajectory laws of Section~\ref{sec:objectives}, define the block
occupancies
\[
    \nu_F(\sigma):=\sum_{\tau:\,\sigma\in\mathcal P(\tau)}p_F(\tau),
    \qquad
    \nu_B(\sigma):=\sum_{\tau:\,\sigma\in\mathcal P(\tau)}p_B(\tau).
\]
In this subsection, $\nu_F,\nu_B$, and the quantities below implicitly
refer to this fixed rule $\mathcal P$. Set
\[
    c_\nu:=\sum_{\sigma\in\mathcal O}
       |\nu_F(\sigma)-\nu_B(\sigma)|,
    \qquad
    \nu(\sigma):=\frac{|\nu_F(\sigma)-\nu_B(\sigma)|}{c_\nu}
\]
when $c_\nu>0$. The zero-denominator conventions are those of
Theorem~\ref{thm:upper_bound_loss}. Since every block contains at least
one edge,
\begin{equation}
    c_\nu\leq
       \mathbb E_{p_F}|\mathcal P(\tau)|
       +\mathbb E_{p_B}|\mathcal P(\tau)|
       \leq\mathbb E_{p_F}n(\tau)+\mathbb E_{p_B}n(\tau).
    \label{eq:subtb-occupancy-mass}
\end{equation}

\begin{proposition}[SubTB loss and terminal mismatch]
\label{prop:subtb-distribution}
Suppose the positive forward and backward policies induce probability
laws $p_F,p_B$ on complete trajectories, $0<Z<\infty$, and
$F(s_0)$ is finite and positive. Assume finite expected block counts and
integrability of $\sum_{\sigma\in\mathcal P(\tau)}|h(\sigma)|$ under
both laws. If $c_\nu=0$, then
\[
    \mathrm{KL}(p_F^\top\|p_R)+\mathrm{KL}(p_R\|p_F^\top)=0.
\]
Otherwise, if $\|\nu/\mu\|_\infty<\infty$ and
$\mathcal L_{\mathcal O}<\infty$, then
\begin{equation}
\begin{aligned}
    &\mathrm{KL}(p_F^\top\|p_R)+\mathrm{KL}(p_R\|p_F^\top)\\
    &\quad\leq\sum_{\sigma\in\mathcal O}
          (\nu_F(\sigma)-\nu_B(\sigma))h(\sigma)\\
    &\quad\leq c_\nu
          \sqrt{2\left\|\frac{\nu}{\mu}\right\|_\infty
                       \mathcal L_{\mathcal O}}.
\end{aligned}
\label{eq:subtb-distribution-bound}
\end{equation}
Without finite coverage, the loss bound is understood as vacuous.
\end{proposition}

\begin{proof}
The additive identity in~\eqref{eq:subtb-additive} and the partition property give
\[
    \sum_{\sigma\in\mathcal P(\tau)}h(\sigma)
       =\log\frac{F(s_0)\prod_{k=0}^{n(\tau)-1}P_F(s_{k+1}\mid s_k)}
                        {R(s_{n(\tau)})
                         \prod_{k=0}^{n(\tau)-1}P_B(s_k\mid s_{k+1})}.
\]
Thus $\log(p_F(\tau)/p_B(\tau))$ is this sum plus
$\log(Z/F(s_0))$. The latter is constant in $\tau$ and cancels in the
symmetric trajectory KL. Integrability permits interchanging trajectory
and block sums, so
\[
    \mathrm{KL}(p_F\|p_B)+\mathrm{KL}(p_B\|p_F)
       =\sum_{\sigma\in\mathcal O}
          (\nu_F(\sigma)-\nu_B(\sigma))h(\sigma).
\]
Lemma~\ref{lemma:KL_bound} gives the first inequality and the case
$c_\nu=0$. For $c_\nu>0$, the last sum is at most
$c_\nu\mathbb E_\nu|h|$. Cauchy--Schwarz and change of measure give the
second inequality.
\end{proof}

Single-edge partitions recover the DB occupancy construction, while a
single block per trajectory recovers the TB construction and gives
$c_\nu\leq2$. Intermediate partitions give the same identity. Summing
all overlapping subtrajectories would count an edge $(s_k,s_{k+1})$
exactly $(k+1)(n-k)$ times, so those occurrence counts cannot replace
block occupancies in this proof. The training objective may nevertheless
average all overlapping subtrajectories through~\eqref{eq:subtb-sampling}.
Different fixed partitions can have different coverage ratios. The
single-edge partition is available for positive all-length weights,
subject to coverage by the trajectory sampling law; the complete-trajectory
partition is natural at the TB endpoint. For restricted or incomplete
sampling, the proposition is informative only when the chosen partition
has finite discrepancy coverage.

\subsection{From SubTB stationarity to optimality}
\label{app:subtb-stationarity}

Use the score, tangent space, residual, and effective Fisher matrix from
Section~\ref{sec:stationary}, now with the SubTB imbalance $h$:
\begin{align*}
    \psi_{\bm\theta}(\sigma)
       &:=\nabla_{\bm\theta}h_{\bm\theta}(\sigma),
    &\mathsf T_{\bm\theta}
       &:=\{\bm u^\top\psi_{\bm\theta}(\cdot):
                   \bm u\in\mathbb R^d\},\\
    r_{\mu,\bm\theta}
       &:=(I-\Pi_{\mathsf T_{\bm\theta}}^\mu)h_{\bm\theta},
    &\mathcal I_\mu(\bm\theta)
       &:=\mathbb E_\mu[\psi_{\bm\theta}\psi_{\bm\theta}^{\top}].
\end{align*}
Let
$\epsilon_\mu:=\sup_{\bm\theta\in\Theta}
\|r_{\mu,\bm\theta}\|_\mu$. Assume finite loss and score second
moments and justify differentiation under the expectation, for example by
local summable domination on a countable object space. These requirements
hold pointwise for finite positive smooth tabular models.

\begin{proposition}[SubTB stationarity bound]
\label{prop:subtb-stationarity}
Under the hypotheses of Proposition~\ref{prop:subtb-distribution}, suppose
$c_\nu>0$, $\|\nu/\mu\|_\infty<\infty$, and
\[
    \mathbb E_\nu\|\psi_{\bm\theta}(\sigma)\|_2\leq M,
    \qquad
    \|\mathcal I_\mu(\bm\theta)^\dagger\|_{2\to2}
       \leq\frac1\lambda.
\]
Then
\begin{equation}
\begin{aligned}
    &\mathrm{KL}(p_F^\top\|p_R)+\mathrm{KL}(p_R\|p_F^\top)\leq c_\nu\left(
       \sqrt{\left\|\frac{\nu}{\mu}\right\|_\infty}\,
               \epsilon_\mu
             +\frac{M}{\lambda}
                \|\nabla_{\bm\theta}\mathcal L_{\mathcal O}\|_2\right).
\end{aligned}
\label{eq:subtb-stationarity-bound}
\end{equation}
\end{proposition}

\begin{proof}
As in Lemma~\ref{lemma:local-coefficient}, the normal equations give
\[
    \bm u_\mu^*(\bm\theta)
       =\mathcal I_\mu(\bm\theta)^\dagger
          \nabla_{\bm\theta}\mathcal L_{\mathcal O},
    \qquad
    \|\bm u_\mu^*(\bm\theta)\|_2
       \leq\lambda^{-1}
          \|\nabla_{\bm\theta}\mathcal L_{\mathcal O}\|_2.
\]
The projection identity is
$\Pi_{\mathsf T_{\bm\theta}}^\mu h_{\bm\theta}
=(\bm u_\mu^*)^\top\psi_{\bm\theta}$, so
$h_{\bm\theta}=r_{\mu,\bm\theta}+(\bm u_\mu^*)^\top\psi_{\bm\theta}$.
Finite coverage transfers the residual from $L_2(\mu)$ to the
occupancy discrepancy. Applying the signed bound in
Proposition~\ref{prop:subtb-distribution} and bounding the two components
separately gives
\[
\begin{aligned}
    \mathrm{KL}(p_F^\top\|p_R)+\mathrm{KL}(p_R\|p_F^\top)
    &\leq c_\nu\mathbb E_\nu|h_{\bm\theta}|\\
    &\leq c_\nu\left(
       \sqrt{\left\|\frac{\nu}{\mu}\right\|_\infty}
          \|r_{\mu,\bm\theta}\|_\mu
       +M\|\bm u_\mu^*\|_2\right).
\end{aligned}
\]
The claim follows from the definition of $\epsilon_\mu$. If $c_\nu=0$,
Proposition~\ref{prop:subtb-distribution} already gives zero terminal
mismatch.
\end{proof}

\begin{proposition}[Tabular zero residual for SubTB]
\label{prop:subtb-zero-residual}
On a finite DAG, independent nonterminal log-state-flows and independent
forward softmax logits give
$h_{\bm\theta}\in\mathsf T_{\bm\theta}$ pointwise at every positive
parameter setting. Therefore $\epsilon_\mu=0$ for every sampling measure
$\mu$. The statement also holds with independent backward coordinates
jointly trained by the balance loss.
\end{proposition}

\begin{proof}
For the current DB edge imbalances, the construction in
Appendix~\ref{app:proof_5} defines $V(x)=0$ on terminal states and
\[
    V(s)=\mathbb E_{s'\sim P_F(\cdot\mid s)}
             [h^{\mathrm{DB}}(s,s')+V(s')],
    \qquad
    A(s,s')=h^{\mathrm{DB}}(s,s')+V(s')-V(s).
\]
The advantage has zero conditional mean under $P_F$. Give the log-state-flow
coordinates direction $V$ and the forward-logit coordinates direction $A$.
This single direction $\bm u$ satisfies
$\bm u^\top\nabla h^{\mathrm{DB}}(s,s')=h^{\mathrm{DB}}(s,s')$
on every edge. Summing along any subtrajectory gives
\[
    h(\sigma)=V(s_i)-V(s_j)+
       \sum_{k=i}^{j-1}A(s_k,s_{k+1})
\]
and, by~\eqref{eq:subtb-additive},
\[
    \bm u^\top\nabla h(\sigma)
       =\sum_{k=i}^{j-1}\bm u^\top
          \nabla h^{\mathrm{DB}}(s_k,s_{k+1})
       =h(\sigma).
\]
The endpoint term $V(s_j)$ remains when the subtrajectory ends at an
internal state. The same statewise construction works for overlapping
segments and different lengths. Independent backward coordinates, if
present, receive direction zero. Thus $h_{\bm\theta}$ belongs pointwise
to $\mathsf T_{\bm\theta}$ and the projection residual vanishes under
every $\mu$.
\end{proof}

\paragraph{Forward-looking SubTB.}
\label{app:subtb-forward-looking}
For the fixed-energy parameterization of \citet{pan2023better}, write
$F(s;\bm\theta)=e^{-E(s)}\widetilde F(s;\bm\theta)$, with
$E(x)=-\log R(x)$ and $\widetilde F(x)=1$ on terminals. Direct
substitution gives
\begin{equation}
    h(\sigma)
       =\log\frac{\widetilde F(s_i;\bm\theta)
                         \prod_{k=i}^{j-1}P_F(s_{k+1}\mid s_k)}
                        {\widetilde F(s_j;\bm\theta)
                         \prod_{k=i}^{j-1}P_B(s_k\mid s_{k+1})}
         +E(s_j)-E(s_i).
    \label{eq:fl-subtb}
\end{equation}
For a single edge this is FL-DB. Along a longer segment the intermediate
energy evaluations cancel, leaving its endpoint difference. The
effective terminal flow remains $R(x)$, so the distributional proof is
unchanged. With $E$ fixed, the independent coordinates
$\log F(s)=\log\widetilde F(s)-E(s)$ differ by a translation with
identity derivative, and Proposition~\ref{prop:subtb-zero-residual}
applies to the translated tabular model.

For a neural log-flow $\widetilde f_{\bm\theta}(s)
=\log\widetilde F(s;\bm\theta)$, the energy term contributes no parameter
derivative. If only the log-flow output is shifted, with the policy maps
and $\mu$ fixed, the score features, tangent space, and Fisher are
unchanged at a fixed parameter. However, the imbalance being projected
changes and need not lie in that tangent space. Along actual training,
the iterates and sampling law can change, so the residual and uniform
conditioning assumptions must still be checked for
$\widetilde f_{\bm\theta}-E$. The given energy is distinct from our
imbalance-dependent tangent witness $V$. Learned or externally changing
energies add derivative or drift terms and require a separate analysis.

\subsection{Convergence of stochastic SubTB training}
\label{app:subtb-convergence}

Let
$\mathcal L_t=\mathcal L_{\mathcal O}(\bm\theta_t,\bm\omega_t;\bm\varphi_t)$.
The balance update differentiates the pointwise SubTB loss while holding
sampling and any separate backward parameters fixed. The following
conditions specify exactly which parts of Section~\ref{sec:convergence-rates}
must be checked for this objective:
\begin{enumerate}
    \item The SubTB oracle satisfies~\eqref{eq:balance-sgd}; the initial
    expected loss is finite, and the loss is jointly smooth on the
    admissible update region as in Assumption~\ref{assump:sgd-smoothness}.
    Loss derivatives can be interchanged with expectations, and the
    required update and score second moments are finite.
    \item Almost surely along training,
    $\|r_{\mu_{\bm\omega_t},\bm\theta_t}\|_{\mu_{\bm\omega_t}}\leq\bar\epsilon$
    and $\|\mathcal I_{\mu_{\bm\omega_t}}(\bm\theta_t)^\dagger\|_{2\to2}
    \leq1/\lambda$, with deterministic uniform constants.
    \item For changing sampling, the SubTB loss satisfies
    Assumptions~\ref{assump:controlled-omega-gradient}
    and~\ref{assump:controlled-sampling}, with
    $C_\omega K_\theta/\sqrt\lambda<1$.
    With a separate backward update, also require
    Assumption~\ref{assump:controlled-varphi-gradient} and use
    Assumption~\ref{assump:controlled-sampling-extension} instead of
    Assumption~\ref{assump:controlled-sampling}. Fixed sampling omits
    the sampling block and has $\rho=0$.
    \item For terminal-distribution conclusions, the hypotheses of
    Proposition~\ref{prop:subtb-distribution} hold at every iterate for
    one fixed partition rule, and
    $c_{\nu,t}\leq\bar B$,
    $\|\nu_t/\mu_{\bm\omega_t}\|_\infty\leq\bar C_\nu$ whenever
    $c_{\nu,t}>0$, with deterministic finite bounds.
\end{enumerate}
In particular, the last item includes integrability of the absolute
block-imbalance sum under both current trajectory laws. Finite DAGs
ensure finite block counts and pointwise integrability, but do not
automatically give a uniform coverage or conditioning bound.

\begin{corollary}[SubTB training rates]
\label{cor:subtb-rates}
Under items 1--3, with fixed or jointly balance-trained $P_B$, suitable
$a,t_0>0$ and $\eta_t=a/(t+t_0)$ give
\[
    \mathbb E\mathcal L_T=O(\bar\epsilon^2+T^{-1}).
\]
If item 4 also holds, then
\begin{equation}
    \mathbb E[\mathrm{KL}(p_{F,T}^\top\|p_R)+\mathrm{KL}(p_R\|p_{F,T}^\top)]
    =O(\bar\epsilon+T^{-1/2}).
\end{equation}
For a separately updated backward policy, define $z_t,d_t$ as in
Section~\ref{sec:convergence-rates}. Under items 1--3,
\begin{equation}
    \mathbb E_t\mathcal L_{t+1}
       \leq(1-c\eta_t)\mathcal L_t
          +C\eta_t\bar\epsilon^2+C\eta_t^2
          +C\frac{z_t^2}{\eta_t}
    \label{eq:subtb-drift-recursion}
\end{equation}
for suitable $c,C>0$. Consequently $d_t=o(\eta_t)$ gives
$\limsup_t\mathbb E\mathcal L_t=O(\bar\epsilon^2)$;
$d_t=O(\eta_t^{3/2})$ preserves the $O(\bar\epsilon^2+T^{-1})$
loss bound; and $d_t=O(t^{-b})$, $1<b<3/2$, gives
$O(\bar\epsilon^2+T^{-2(b-1)})$.
Under item 4, these imply, respectively, an $O(\bar\epsilon)$ terminal
neighborhood, an $O(\bar\epsilon+T^{-1/2})$ terminal bound, and an
$O(\bar\epsilon+T^{-(b-1)})$ terminal bound.
\end{corollary}

\begin{proof}
Projection and Pythagoras give
\[
    2\mathcal L_t
       =\|r_{\mu_{\bm\omega_t},\bm\theta_t}\|_{\mu_{\bm\omega_t}}^2
          +(\nabla_{\bm\theta}\mathcal L_t)^\top
             \mathcal I_{\mu_{\bm\omega_t}}(\bm\theta_t)^\dagger
                \nabla_{\bm\theta}\mathcal L_t
       \leq\bar\epsilon^2+
             \lambda^{-1}\|\nabla_{\bm\theta}\mathcal L_t\|_2^2.
\]
The common descent proof in Appendix~\ref{app:convergence-geometry}
uses only this inequality, the unbiased oracle, joint smoothness, and
the specified sensitivity and update bounds. Applying it to
$\mathcal L_{\mathcal O}$ proves~\eqref{eq:subtb-drift-recursion},
with $z_t=0$ when no separate backward update is present. The scalar
recursion arguments in Appendix~\ref{app:convergence-proofs} give the
stated rates with the same stepsize choices. Finally,
Proposition~\ref{prop:subtb-distribution} and Jensen's inequality yield
$\mathbb E[\mathrm{KL}(p_{F,T}^\top\|p_R)+\mathrm{KL}(p_R\|p_{F,T}^\top)]
    \leq\bar B\sqrt{2\bar C_\nu\mathbb E\mathcal L_T}.$
\end{proof}

\paragraph{Length moments and the limits of the extension.}
For a separately parameterized backward policy,
\[
    \nabla_{\bm\varphi}h(\sigma)
       =-\sum_{k=i}^{j-1}
          \nabla_{\bm\varphi}\log P_B(s_k\mid s_{k+1};\bm\varphi).
\]
If each edge score has norm at most $S_B$, then
$\mathbb E_{\mu_{\bm\omega_t}}\|\nabla_{\bm\varphi}h\|_2^2
\leq S_B^2\mathbb E_{\mu_{\bm\omega_t}}\ell(\sigma)^2$.
A uniform second-moment bound on sampled subtrajectory lengths therefore
supplies the backward sensitivity constant through Cauchy--Schwarz.
Analogous bounded endpoint log-flow scores and forward edge scores
control the balance-score second moment. These moments concern the
training law $\mu_{\bm\omega_t}$; they are distinct from block-count and
integrability conditions under $p_{F,t},p_{B,t}$. Bounded imbalance and
sampling scores give the sampling sensitivity bound as in
Appendix~\ref{app:convergence-regularity}. On countable spaces, local
summable domination of derivatives is one sufficient way to justify
the required interchanges of differentiation and expectation.

Changing $\gamma$ affects coverage, conditioning, and noise through
$\mu_{\bm\omega_t}$; these constants need not be uniform in $\gamma$. Tabular zero
residual alone does not ensure convergence. Fixed-energy FL-SubTB
in~\eqref{eq:fl-subtb} obeys the same conditional guarantees, and replay
and Polyak targets retain the qualifications of
Appendix~\ref{app:sampling-stabilization}.

\end{document}